\documentclass{article}
\usepackage{verbatim}
\usepackage{pifont}
\newcommand{\cmark}{\ding{51}} % ✓
\newcommand{\xmark}{\ding{55}} % ✗

\PassOptionsToPackage{authoryear, round}{natbib}
\usepackage[main, final]{neurips_2026}

\usepackage[utf8]{inputenc} % allow utf-8 input
\usepackage[T1]{fontenc}    % use 8-bit T1 fonts
\usepackage{hyperref}  
\usepackage{etoc}
\usepackage{url}            % simple URL typesetting
\usepackage{booktabs}       % professional-quality tables
\usepackage{amsfonts}       % blackboard math symbols
\usepackage{nicefrac}       % compact symbols for 1/2, etc.
\usepackage{microtype}      % microtypography
\usepackage{xcolor}         % colors
\usepackage{amsmath}
\usepackage{graphicx}
\usepackage{comment}
\usepackage{subcaption} % in the preamble
\usepackage{booktabs}
\usepackage{wrapfig}
\usepackage{cleveref}

\crefname{figure}{Fig.}{Figs.}          \Crefname{figure}{Fig.}{Figs.}
\crefname{subfigure}{Fig.}{Figs.}       \Crefname{subfigure}{Fig.}{Figs.}
\crefname{table}{Tab.}{Tabs.}           \Crefname{table}{Tab.}{Tabs.}
\crefname{section}{Sec.}{Secs.}         \Crefname{section}{Sec.}{Secs.}
\crefname{subsection}{Sec.}{Secs.}      \Crefname{subsection}{Sec.}{Secs.}
\crefname{subsubsection}{Sec.}{Secs.}   \Crefname{subsubsection}{Sec.}{Secs.}
\crefname{appendix}{App.}{Apps.}        \Crefname{appendix}{App.}{Apps.}
\crefname{equation}{Eq.}{Eqs.}          \Crefname{equation}{Eq.}{Eqs.}
\crefname{algorithm}{Alg.}{Algs.}       \Crefname{algorithm}{Alg.}{Algs.}
\crefname{theorem}{Thm.}{Thms.}         \Crefname{theorem}{Thm.}{Thms.}
\crefname{proposition}{Prop.}{Props.}   \Crefname{proposition}{Prop.}{Props.}

\usepackage{amsthm}
\usepackage{algorithm}
\usepackage{algpseudocode}
\usepackage{float}  
\usepackage[table]{xcolor}
\usepackage[most]{tcolorbox}

\newtheorem{theorem}{Theorem}
\newtheorem{proposition}{Proposition}

\newcommand{\efflayer}{effective layer\xspace}
\newcommand{\efflayers}{effective layers\xspace}

\definecolor{brightgreen}{HTML}{00FF00}
\definecolor{textgray}{gray}{0.5}
\definecolor{fprmblue}{HTML}{648FFF}
\definecolor{cornflowerblue}{rgb}{0.34, 0.51, 0.82}
\colorlet{fprmbluelight}{fprmblue!20}

\definecolor{light}{RGB}{240, 240, 240} % RGB values 

\newtcolorbox{lightgraybox}{
    enhanced,
    frame hidden,
    sharp corners,
    colback=fprmbluelight,
    top=2pt,
    bottom=2pt,
}

\newtcolorbox{lightgrayframe}{
    enhanced,
    sharp corners,
    colback=white,      % no grey fill (white = transparent on a white page)
    colframe=fprmbluelight,     % frame color
    boxrule=3pt,      % frame thickness
    top=2pt,
    bottom=2pt,
}

\hypersetup{
    colorlinks=true,
    linkcolor=cornflowerblue,   % internal links (sections, figures, eqs)
    filecolor=magenta,
    urlcolor=teal,              % external URLs
    citecolor=cornflowerblue,   % \cite links
}

\title{Fixed-Point Reasoners:\\ Stable and Adaptive Deep Looped Transformers}

\author{%
  \textbf{Sajad Movahedi}\thanks{Equal contribution.}~~\textsuperscript{1} \quad
  \textbf{Vera Milovanović}\footnotemark[1]~~\textsuperscript{1,2} \quad
  \textbf{Shlomo Libo Feigin}\footnotemark[1]~~\textsuperscript{1,2} \quad
  \textbf{Alexander Theus}\footnotemark[1]~~\textsuperscript{1, 2} \\[0.4em]
  \textbf{Thomas Hofmann}\textbf{\textsuperscript{2}} \quad
  \textbf{Valentina Boeva}\thanks{Equal advising.}~~\textbf{\textsuperscript{2, 3, 4}} \quad
  \textbf{T. Konstantin Rusch}\footnotemark[2]~~\textbf{\textsuperscript{1,5}} \quad
  \textbf{Antonio Orvieto}\footnotemark[2]~~\textbf{\textsuperscript{1}} \\[0.6em]
  \textsuperscript{1}ELLIS Institute Tübingen, Max Planck Institute for Intelligent Systems, Tübingen AI Center \\
  \textsuperscript{2}ETH Zurich \quad
  \textsuperscript{3}Swiss Institute of Bioinformatics \quad
  \textsuperscript{4}Université Paris Cité \quad
  \textsuperscript{5}Liquid AI \\[0.4em]
  \texttt{\{sajad.movahedi, vera.milovanovic\}@tue.ellis.eu}
}

\tcbset{algbox/.style={enhanced, sharp corners, frame hidden, colback=white,
  boxsep=0pt, left=0pt, right=0pt, top=2pt, bottom=2pt,
  borderline north={0.8pt}{0pt}{black},
  borderline south={0.4pt}{0pt}{black}}}
\newcommand{\algcaption}[1]{%
  \captionsetup{style=ruled, margin=0pt, singlelinecheck=false, skip=0pt, position=top}%
  \captionof{algorithm}{#1}\par\kern2pt\hrule\kern2pt}

\usepackage{amssymb}
\usepackage{amsfonts}
\usepackage{mathrsfs}
\usepackage{xspace}
\usepackage{bm}
\usepackage{upgreek}

\newcommand{\safemath}[2]{\newcommand{#1}{\ensuremath{#2}\xspace}}

\safemath{\bma}{\mathbf{a}}
\safemath{\bmb}{\mathbf{b}}
\safemath{\bmc}{\mathbf{c}}
\safemath{\bmd}{\mathbf{d}}
\safemath{\bme}{\mathbf{e}}
\safemath{\bmf}{\mathbf{f}}
\safemath{\bmg}{\mathbf{g}}
\safemath{\bmh}{\mathbf{h}}
\safemath{\bmi}{\mathbf{i}}
\safemath{\bmj}{\mathbf{j}}
\safemath{\bmk}{\mathbf{k}}
\safemath{\bml}{\mathbf{l}}
\safemath{\bmm}{\mathbf{m}}
\safemath{\bmn}{\mathbf{n}}
\safemath{\bmo}{\mathbf{o}}
\safemath{\bmp}{\mathbf{p}}
\safemath{\bmq}{\mathbf{q}}
\safemath{\bmr}{\mathbf{r}}
\safemath{\bms}{\mathbf{s}}
\safemath{\bmt}{\mathbf{t}}
\safemath{\bmu}{\mathbf{u}}
\safemath{\bmv}{\mathbf{v}}
\safemath{\bmw}{\mathbf{w}}
\safemath{\bmx}{\mathbf{x}}
\safemath{\bmy}{\mathbf{y}}
\safemath{\bmz}{\mathbf{z}}
\safemath{\bmzero}{\mathbf{0}}
\safemath{\bmone}{\mathbf{1}}

\safemath{\bmbeta}{\mathbf{beta}}

\bmdefine{\biad}{a}
\bmdefine{\bibd}{b}
\bmdefine{\bicd}{c}
\bmdefine{\bidd}{d}
\bmdefine{\bied}{e}
\bmdefine{\bifd}{f}
\bmdefine{\bigd}{g}
\bmdefine{\bihd}{h}
\bmdefine{\biid}{i}
\bmdefine{\bijd}{j}
\bmdefine{\bikd}{k}
\bmdefine{\bild}{l}
\bmdefine{\bimd}{m}
\bmdefine{\bind}{n}
\bmdefine{\biod}{o}
\bmdefine{\bipd}{p}
\bmdefine{\biqd}{q}
\bmdefine{\bird}{r}
\bmdefine{\bisd}{s}
\bmdefine{\bitd}{t}
\bmdefine{\biud}{u}
\bmdefine{\bivd}{v}
\bmdefine{\biwd}{w}
\bmdefine{\bixd}{x}
\bmdefine{\biyd}{y}
\bmdefine{\bizd}{z}

\bmdefine{\bixid}{\xi}
\bmdefine{\bilambdad}{\lambda}
\bmdefine{\bimud}{\mu}
\bmdefine{\bithetad}{\theta}
\bmdefine{\biphid}{\phi}

\safemath{\bmia}{\biad}
\safemath{\bmib}{\bibd}
\safemath{\bmic}{\bicd}
\safemath{\bmid}{\bidd}
\safemath{\bmie}{\bied}
\safemath{\bmif}{\bifd}
\safemath{\bmig}{\bigd}
\safemath{\bmih}{\bihd}
\safemath{\bmii}{\biid}
\safemath{\bmij}{\bijd}
\safemath{\bmik}{\bikd}
\safemath{\bmil}{\bild}
\safemath{\bmim}{\bimd}
\safemath{\bmin}{\bind}
\safemath{\bmio}{\biod}
\safemath{\bmip}{\bipd}
\safemath{\bmiq}{\biqd}
\safemath{\bmir}{\bird}
\safemath{\bmis}{\bisd}
\safemath{\bmit}{\bitd}
\safemath{\bmiu}{\biud}
\safemath{\bmiv}{\bivd}
\safemath{\bmiw}{\biwd}
\safemath{\bmix}{\bixd}
\safemath{\bmiy}{\biyd}
\safemath{\bmiz}{\bizd}

\safemath{\bmxi}{\bixid}
\safemath{\bmlambda}{\bilambdad}
\safemath{\bmmu}{\bimud}
\safemath{\bmtheta}{\bithetad}
\safemath{\bmphi}{\biphid}

\safemath{\bA}{\mathbf{A}}
\safemath{\bB}{\mathbf{B}}
\safemath{\bC}{\mathbf{C}}
\safemath{\bD}{\mathbf{D}}
\safemath{\bE}{\mathbf{E}}
\safemath{\bF}{\mathbf{F}}
\safemath{\bG}{\mathbf{G}}
\safemath{\bH}{\mathbf{H}}
\safemath{\bI}{\mathbf{I}}
\safemath{\bJ}{\mathbf{J}}
\safemath{\bK}{\mathbf{K}}
\safemath{\bL}{\mathbf{L}}
\safemath{\bM}{\mathbf{M}}
\safemath{\bN}{\mathbf{N}}
\safemath{\bO}{\mathbf{O}}
\safemath{\bP}{\mathbf{P}}
\safemath{\bQ}{\mathbf{Q}}
\safemath{\bR}{\mathbf{R}}
\safemath{\bS}{\mathbf{S}}
\safemath{\bT}{\mathbf{T}}
\safemath{\bU}{\mathbf{U}}
\safemath{\bV}{\mathbf{V}}
\safemath{\bW}{\mathbf{W}}
\safemath{\bX}{\mathbf{X}}
\safemath{\bY}{\mathbf{Y}}
\safemath{\bZ}{\mathbf{Z}}

\safemath{\bZero}{\mathbf{0}}
\safemath{\bOne}{\mathbf{1}}
\safemath{\bDelta}{\mathbf{\Delta}}
\safemath{\bLambda}{\mathbf{\UpLambda}}
\safemath{\bPhi}{\mathbf{\Upphi}}
\safemath{\bSigma}{\mathbf{\Upsigma}}
\safemath{\bOmega}{\mathbf{\Upomega}}
\safemath{\bTheta}{\mathbf{\Uptheta}}

\bmdefine{\biAd}{A}
\bmdefine{\biBd}{B}
\bmdefine{\biCd}{C}
\bmdefine{\biDd}{D}
\bmdefine{\biEd}{E}
\bmdefine{\biFd}{F}
\bmdefine{\biGd}{G}
\bmdefine{\biHd}{H}
\bmdefine{\biId}{I}
\bmdefine{\biJd}{J}
\bmdefine{\biKd}{K}
\bmdefine{\biLd}{L}
\bmdefine{\biMd}{M}
\bmdefine{\biOd}{N}
\bmdefine{\biPd}{O}
\bmdefine{\biQd}{P}
\bmdefine{\biRd}{R}
\bmdefine{\biSd}{S}
\bmdefine{\biTd}{T}
\bmdefine{\biUd}{U}
\bmdefine{\biVd}{V}
\bmdefine{\biWd}{W}
\bmdefine{\biXd}{X}
\bmdefine{\biYd}{Y}
\bmdefine{\biZd}{Z}

\bmdefine{\biDelta}{\Delta}
\bmdefine{\biLambda}{\Lambda}
\bmdefine{\biPhi}{\Phi}
\bmdefine{\biSigma}{\Sigma}
\bmdefine{\biOmega}{\Omega}
\bmdefine{\biTheta}{\Theta}

\safemath{\bimA}{\biAd}
\safemath{\bimB}{\biBd}
\safemath{\bimC}{\biCd}
\safemath{\bimD}{\biDd}
\safemath{\bimE}{\biEd}
\safemath{\bimF}{\biFd}
\safemath{\bimG}{\biGd}
\safemath{\bimH}{\biHd}
\safemath{\bimI}{\biId}
\safemath{\bimJ}{\biJd}
\safemath{\bimK}{\biKd}
\safemath{\bimL}{\biLd}
\safemath{\bimM}{\biMd}
\safemath{\bimN}{\biNd}
\safemath{\bimO}{\biOd}
\safemath{\bimP}{\biPd}
\safemath{\bimQ}{\biQd}
\safemath{\bimR}{\biRd}
\safemath{\bimS}{\biSd}
\safemath{\bimT}{\biTd}
\safemath{\bimU}{\biUd}
\safemath{\bimV}{\biVd}
\safemath{\bimW}{\biWd}
\safemath{\bimX}{\biXd}
\safemath{\bimY}{\biYd}
\safemath{\bimZ}{\biZd}

\safemath{\bimDelta}{\biDelta}
\safemath{\bimLambda}{\biLambda}
\safemath{\bimPhi}{\biPhi}
\safemath{\bimSigma}{\biSigma}
\safemath{\bimOmega}{\biOmega}
\safemath{\bimTheta}{\biTheta}

\safemath{\setA}{\mathcal{A}}
\safemath{\setB}{\mathcal{B}}
\safemath{\setC}{\mathcal{C}}
\safemath{\setD}{\mathcal{D}}
\safemath{\setE}{\mathcal{E}}
\safemath{\setF}{\mathcal{F}}
\safemath{\setG}{\mathcal{G}}
\safemath{\setH}{\mathcal{H}}
\safemath{\setI}{\mathcal{I}}
\safemath{\setJ}{\mathcal{J}}
\safemath{\setK}{\mathcal{K}}
\safemath{\setL}{\mathcal{L}}
\safemath{\setM}{\mathcal{M}}
\safemath{\setN}{\mathcal{N}}
\safemath{\setO}{\mathcal{O}}
\safemath{\setP}{\mathcal{P}}
\safemath{\setQ}{\mathcal{Q}}
\safemath{\setR}{\mathcal{R}}
\safemath{\setS}{\mathcal{S}}
\safemath{\setT}{\mathcal{T}}
\safemath{\setU}{\mathcal{U}}
\safemath{\setV}{\mathcal{V}}
\safemath{\setW}{\mathcal{W}}
\safemath{\setX}{\mathcal{X}}
\safemath{\setY}{\mathcal{Y}}
\safemath{\setZ}{\mathcal{Z}}
\safemath{\emptySet}{\varnothing}

\safemath{\colA}{\mathscr{A}}
\safemath{\colB}{\mathscr{B}}
\safemath{\colC}{\mathscr{C}}
\safemath{\colD}{\mathscr{D}}
\safemath{\colE}{\mathscr{E}}
\safemath{\colF}{\mathscr{F}}
\safemath{\colG}{\mathscr{G}}
\safemath{\colH}{\mathscr{H}}
\safemath{\colI}{\mathscr{I}}
\safemath{\colJ}{\mathscr{J}}
\safemath{\colK}{\mathscr{K}}
\safemath{\colL}{\mathscr{L}}
\safemath{\colM}{\mathscr{M}}
\safemath{\colN}{\mathscr{N}}
\safemath{\colO}{\mathscr{O}}
\safemath{\colP}{\mathscr{P}}
\safemath{\colQ}{\mathscr{Q}}
\safemath{\colR}{\mathscr{R}}
\safemath{\colS}{\mathscr{S}}
\safemath{\colT}{\mathscr{T}}
\safemath{\colU}{\mathscr{U}}
\safemath{\colV}{\mathscr{V}}
\safemath{\colW}{\mathscr{W}}
\safemath{\colX}{\mathscr{X}}
\safemath{\colY}{\mathscr{Y}}
\safemath{\colZ}{\mathscr{Z}}

\safemath{\opA}{\mathbb{A}}
\safemath{\opB}{\mathbb{B}}
\safemath{\opC}{\mathbb{C}}
\safemath{\opD}{\mathbb{D}}
\safemath{\opE}{\mathbb{E}}
\safemath{\opF}{\mathbb{F}}
\safemath{\opG}{\mathbb{G}}
\safemath{\opH}{\mathbb{H}}
\safemath{\opI}{\mathbb{I}}
\safemath{\opJ}{\mathbb{J}}
\safemath{\opK}{\mathbb{K}}
\safemath{\opL}{\mathbb{L}}
\safemath{\opM}{\mathbb{M}}
\safemath{\opN}{\mathbb{N}}
\safemath{\opO}{\mathbb{O}}
\safemath{\opP}{\mathbb{P}}
\safemath{\opQ}{\mathbb{Q}}
\safemath{\opR}{\mathbb{R}}
\safemath{\opS}{\mathbb{S}}
\safemath{\opT}{\mathbb{T}}
\safemath{\opU}{\mathbb{U}}
\safemath{\opV}{\mathbb{V}}
\safemath{\opW}{\mathbb{W}}
\safemath{\opX}{\mathbb{X}}
\safemath{\opY}{\mathbb{Y}}
\safemath{\opZ}{\mathbb{Z}}
\safemath{\opZero}{\mathbb{O}}
\safemath{\identityop}{\opI}

\safemath{\veca}{\bma}
\safemath{\vecb}{\bmb}
\safemath{\vecc}{\bmc}
\safemath{\vecd}{\bmd}
\safemath{\vece}{\bme}
\safemath{\vecf}{\bmf}
\safemath{\vecg}{\bmg}
\safemath{\vech}{\bmh}
\safemath{\veci}{\bmi}
\safemath{\vecj}{\bmj}
\safemath{\veck}{\bmk}
\safemath{\vecl}{\bml}
\safemath{\vecm}{\bmm}
\safemath{\vecn}{\bmn}
\safemath{\veco}{\bmo}
\safemath{\vecp}{\bmmp}
\safemath{\vecq}{\bmq}
\safemath{\vecr}{\bmr}
\safemath{\vecs}{\bms}
\safemath{\vect}{\bmt}
\safemath{\vecu}{\bmu}
\safemath{\vecv}{\bmv}
\safemath{\vecw}{\bmw}
\safemath{\vecx}{\bmx}
\safemath{\vecy}{\bmy}
\safemath{\vecz}{\bmz}

\safemath{\veczero}{\bmzero}
\safemath{\vecone}{\bmone}
\safemath{\vecxi}{\bmxi}
\safemath{\veclambda}{\bmlambda}
\safemath{\vecmu}{\bmmu}
\safemath{\vectheta}{\bmtheta}
\safemath{\vecphi}{\bmphi}

\safemath{\matA}{\bA}
\safemath{\matB}{\bB}
\safemath{\matC}{\bC}
\safemath{\matD}{\bD}
\safemath{\matE}{\bE}
\safemath{\matF}{\bF}
\safemath{\matG}{\bG}
\safemath{\matH}{\bH}
\safemath{\matI}{\bI}
\safemath{\matJ}{\bJ}
\safemath{\matK}{\bK}
\safemath{\matL}{\bL}
\safemath{\matM}{\bM}
\safemath{\matN}{\bN}
\safemath{\matO}{\bO}
\safemath{\matP}{\bP}
\safemath{\matQ}{\bQ}
\safemath{\matR}{\bR}
\safemath{\matS}{\bS}
\safemath{\matT}{\bT}
\safemath{\matU}{\bU}
\safemath{\matV}{\bV}
\safemath{\matW}{\bW}
\safemath{\matX}{\bX}
\safemath{\matY}{\bY}
\safemath{\matZ}{\bZ}
\safemath{\matzero}{\bmzero}

\safemath{\matDelta}{\bDelta}
\safemath{\matLambda}{\bLambda}
\safemath{\matPhi}{\bPhi}
\safemath{\matSigma}{\bSigma}
\safemath{\matOmega}{\bOmega}
\safemath{\matTheta}{\bTheta}

\safemath{\matidentity}{\matI}
\safemath{\matone}{\matO}

\safemath{\rnda}{A}
\safemath{\rndb}{B}
\safemath{\rndc}{C}
\safemath{\rndd}{D}
\safemath{\rnde}{E}
\safemath{\rndf}{F}
\safemath{\rndg}{G}
\safemath{\rndh}{H}
\safemath{\rndi}{I}
\safemath{\rndj}{J}
\safemath{\rndk}{K}
\safemath{\rndl}{L}
\safemath{\rndm}{M}
\safemath{\rndn}{N}
\safemath{\rndo}{O}
\safemath{\rndp}{P}
\safemath{\rndq}{Q}
\safemath{\rndr}{R}
\safemath{\rnds}{S}
\safemath{\rndt}{T}
\safemath{\rndu}{U}
\safemath{\rndv}{V}
\safemath{\rndw}{W}
\safemath{\rndx}{X}
\safemath{\rndy}{Y}
\safemath{\rndz}{Z}

\safemath{\rveca}{\bimA}
\safemath{\rvecb}{\bimB}
\safemath{\rvecc}{\bimC}
\safemath{\rvecd}{\bimD}
\safemath{\rvece}{\bimE}
\safemath{\rvecf}{\bimF}
\safemath{\rvecg}{\bimG}
\safemath{\rvech}{\bimH}
\safemath{\rveci}{\bimI}
\safemath{\rvecj}{\bimJ}
\safemath{\rveck}{\bimK}
\safemath{\rvecl}{\bimL}
\safemath{\rvecm}{\bimM}
\safemath{\rvecn}{\bimN}
\safemath{\rveco}{\bomO}
\safemath{\rvecp}{\bimP}
\safemath{\rvecq}{\bimQ}
\safemath{\rvecr}{\bimR}
\safemath{\rvecs}{\bimS}
\safemath{\rvect}{\bimT}
\safemath{\rvecu}{\bimU}
\safemath{\rvecv}{\bimV}
\safemath{\rvecw}{\bimW}
\safemath{\rvecx}{\bimX}
\safemath{\rvecy}{\bimY}
\safemath{\rvecz}{\bimZ}

\safemath{\rvecxi}{\bmxi}
\safemath{\rveclambda}{\bmlambda}
\safemath{\rvecmu}{\bmmu}
\safemath{\rvectheta}{\bmtheta}
\safemath{\rvecphi}{\bmphi}

\safemath{\rmatA}{\bimA}
\safemath{\rmatB}{\bimB}
\safemath{\rmatC}{\bimC}
\safemath{\rmatD}{\bimD}
\safemath{\rmatE}{\bimE}
\safemath{\rmatF}{\bimF}
\safemath{\rmatG}{\bimG}
\safemath{\rmatH}{\bimH}
\safemath{\rmatI}{\bimI}
\safemath{\rmatJ}{\bimJ}
\safemath{\rmatK}{\bimK}
\safemath{\rmatL}{\bimL}
\safemath{\rmatM}{\bimM}
\safemath{\rmatN}{\bimN}
\safemath{\rmatO}{\bimO}
\safemath{\rmatP}{\bimP}
\safemath{\rmatQ}{\bimQ}
\safemath{\rmatR}{\bimR}
\safemath{\rmatS}{\bimS}
\safemath{\rmatT}{\bimT}
\safemath{\rmatU}{\bimU}
\safemath{\rmatV}{\bimV}
\safemath{\rmatW}{\bimW}
\safemath{\rmatX}{\bimX}
\safemath{\rmatY}{\bimY}
\safemath{\rmatZ}{\bimZ}

\safemath{\rmatDelta}{\bimDelta}
\safemath{\rmatLambda}{\bimLambda}
\safemath{\rmatPhi}{\bimPhi}
\safemath{\rmatSigma}{\bimSigma}
\safemath{\rmatOmega}{\bimOmega}
\safemath{\rmatTheta}{\bimTheta}
\usepackage{amssymb}
\usepackage{amsfonts}
\usepackage{mathrsfs}
\usepackage{xspace}
\usepackage{bm}

\usepackage{scalerel}

\newenvironment{textbmatrix}{	\setlength{\arraycolsep}{2.5pt}%
								\big[\begin{matrix}}{\end{matrix}\big]%
								\raisebox{0.08ex}{\vphantom{M}}}

\def\be{\begin{equation}}
\def\ee{\end{equation}}
\def\een{\nonumber \end{equation}}
\def\mat{\begin{bmatrix}}
\def\emat{\end{bmatrix}}
\def\btm{\begin{textbmatrix}}
\def\etm{\end{textbmatrix}}

\def\ba#1\ea{\begin{align}#1\end{align}}
\def\bas#1\eas{\begin{align*}#1\end{align*}}
\def\bs#1\es{\begin{split}#1\end{split}} 
\def\bg#1\eg{\begin{gather}#1\end{gather}} 
\def\bi#1\ei{\begin{itemize}#1\end{itemize}}

\safemath{\dirac}{\delta}					% Dirac delta
\safemath{\krond}{\dirac}					% Kronecker delta
\safemath{\upto}{\uparrow}
\safemath{\downto}{\downarrow}
\safemath{\iu}{j}							% imaginary unit
\safemath{\ev}{\lambda}						% eigenvalue
\safemath{\hilseqspace}{l^{2}}				% Hilbert sequence space
\newcommand{\banachfunspace}[1]{\setL^{#1}}	% Banach function space
\safemath{\hilfunspace}{\banachfunspace{2}}	% Hilbert function space
\safemath{\SNR}{\text{\sc snr}} 				% signal to noise ratio
\safemath{\No}{N_0}							% noise spectral density
\safemath{\Es}{E_s}							% energy per symbol
\safemath{\Eb}{E_b}							% energy per bit
\safemath{\EbNo}{\frac{\Eb}{\No}}
\safemath{\EsNo}{\frac{\Es}{\No}}

\DeclareMathOperator{\CHop}{\ensuremath{\opH}} % channel operator
\safemath{\tvir}{\rndh_{\CHop}}				% time-varying impulse response
\safemath{\tvtf}{\rndl_{\CHop}}				% 	-''- transfer function
\safemath{\spf}{\rnds_{\CHop}}				% spreading function
\safemath{\bff}{H_{\CHop}}					% bi-freuqency function

\safemath{\ircf}{r_{h}}						% impulse response correlation fn.
\safemath{\tftvcf}{r_{s}}					% scattering function
\safemath{\tfcf}{r_{l}}						% time-frequency correlation fn.
\safemath{\bfcf}{r_{H}}						% bi-frequency correlation fn.

\safemath{\tcorr}{c_h}						% time-correlation function
\safemath{\scf}{c_{s}}						% spreading function
\safemath{\tfcorr}{c_{l}}					% transfer-function correlation
\safemath{\fcorr}{c_{H}}						% frequency-correlation function

\safemath{\mi}{I}							% mutual information
\safemath{\capacity}{C}						% capacity

\safemath{\normal}{\mathcal{N}}			% normal distribution
\safemath{\jpg}{\mathcal{CN}}			% jointly proper Gaussian
\safemath{\mchain}{\leftrightarrow}		% Markov chain
\safemath{\dB}{\,\mathrm{dB}}
\safemath{\dBm}{\,\mathrm{dBm}}
\safemath{\Hz}{\,\mathrm{Hz}}
\safemath{\kHz}{\,\mathrm{kHz}}
\safemath{\MHz}{\,\mathrm{MHz}}
\safemath{\GHz}{\,\mathrm{GHz}}
\safemath{\s}{\,\mathrm{s}}
\safemath{\ms}{\,\mathrm{ms}}
\safemath{\mus}{\,\mathrm{\mu s}}
\safemath{\ns}{\,\mathrm{ns}}
\safemath{\meter}{\,\mathrm{m}}
\safemath{\mm}{\,\mathrm{mm}}
\safemath{\cm}{\,\mathrm{cm}}
\safemath{\m}{\,\mathrm{m}}
\safemath{\W}{\,\mathrm{W}}
\safemath{\J}{\,\mathrm{J}}
\safemath{\K}{\,\mathrm{K}}
\safemath{\bit}{\,\mathrm{bit}}

\safemath{\define}{=}			% definition

\safemath{\equivalent}{\sim}
\safemath{\distas}{\sim}					% distributed according to
\safemath{\sdiff}{\Delta}				% symmetric set difference

\safemath{\reals}{\mathbb{R}}
\safemath{\positivereals}{\reals_{+}}
\safemath{\integers}{\mathbb{Z}}
\safemath{\posint}{\integers_{+}}
\safemath{\naturals}{\mathbb{N}}
\safemath{\posnaturals}{\naturals_{+}}
\safemath{\complexset}{\mathbb{C}}
\safemath{\rationals}{\mathbb{Q}}

\begin{document}

\maketitle

\vspace{-12pt}

\begin{abstract}
Looped architectures provide an inductive bias toward learning step-by-step procedures for tasks that require compositional reasoning. Increasing the \efflayers of looped architectures can improve their reasoning capabilities. Similar to deep architectures, looped architectures are prone to a signal propagation problem induced by depth as the halting decision is postponed. In this paper, we address the signal propagation issue by using pre-norm layers and residual scaling. By improving signal propagation, we achieve competitive performance without hierarchical looping, resulting in a simpler architecture than previous state-of-the-art models evaluated on puzzle reasoning tasks. Building on these modifications, we propose \textbf{FPRM}: a Transformer-based \textbf{F}ixed-\textbf{P}oint \textbf{R}easoning \textbf{M}odel that uses fixed-point convergence as an end-to-end halting mechanism in a looped architecture. We show that fixed-point halting allows FPRM to adapt its compute to the difficulty of the task. FPRM proves effective on common reasoning benchmarks, namely sudoku, maze, state-tracking and ARC-AGI. The implementation can be found \href{https://github.com/nilskiKonjIzDunava/fprm}{here} and the checkpoints of the trained FPRM models \href{https://huggingface.co/fixed-point-reasoners/fprm}{here}. \looseness=-1
\end{abstract}
\vspace{-12pt}

\section{Introduction}

\begin{wrapfigure}[18]{r}{0.42\textwidth}
    \vspace{-4.2em}
    \centering
    \includegraphics[width=\linewidth]{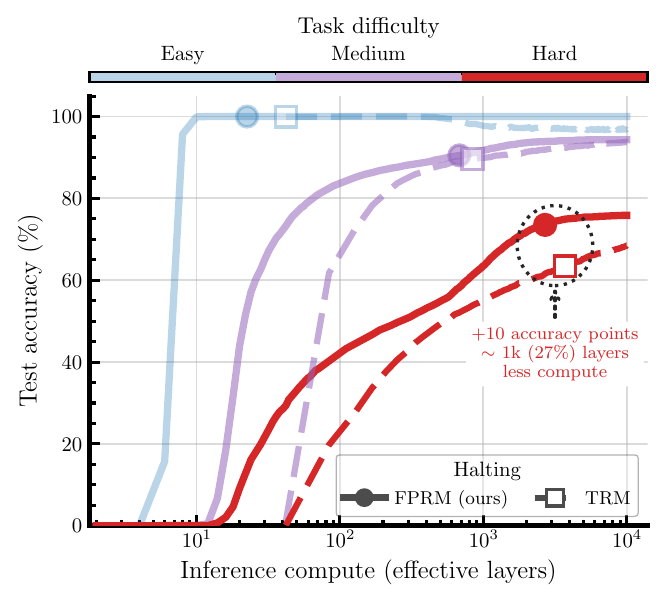}
    \vspace{-1.5em}
    \caption{\textbf{Signal propagation and adaptivity, FPRM vs. TRM:} Sudoku-Extreme performance as a function of compute across difficulty levels. Despite being non-hierarchical, FPRM scales better, and identifies accuracy plateaus more effectively by using fixed-point halting.}
    \label{fig:fprm_trm_halting_points}
    \vspace{-5em}
\end{wrapfigure}

Reasoning in neural networks has increasingly been framed as a problem of scaling test-time compute:  a model should be able to spend more computation on inputs it finds harder~\citep{openai2024o1, DBLP:journals/corr/abs-2408-03314}. However, doing so requires two ingredients. \textbf{(1)} \textbf{flexibility:} the possibility of spending a variable amount of compute on the problem. Once the model is flexible, the next step is \textbf{(2)} \textbf{adaptivity:} a way to scale the compute spent on the problem; i.e., when to halt the computation.

The standard way to achieve both is through a Chain-of-Thought (CoT) mechanism~\citep{DBLP:conf/nips/Wei0SBIXCLZ22}. With CoT, the model scales compute through verbalization, and makes halting decisions  based on predicting a specialized halting token. However, this emerging behavior requires a special training regime and hand-crafted reasoning traces~\citep{DBLP:journals/nature/GuoYZSWZXZMBZY025}. This makes the method complex and undermines the desirable property of end-to-end training.

An alternative approach to end-to-end training for reasoning models is emerging in the form of looped architectures~\citep{dehghani2019universaltransformers, DBLP:conf/nips/BansalSBEHGG22, saunshi2025reasoninglatentthoughtspower, hao2025traininglargelanguagemodels, geiping2025scalingtesttimecomputelatent}. Whereas in CoT, the compute scales along the sequence dimension, in looped architectures it scales along the depth dimension ($i$):
\begin{equation}
    \label{eqn:looping}
    \bmz_{i+1}=f_\theta(\bmz_i ; \bmx),
\end{equation}
where $f_\theta(\bmz;\bmx)$ is a neural network, $\bmx$ is the input, and $\bmz_i$ is the $i^{\text{th}}$ latent representation. Thus, compute can be increased at test-time by looping in depth, naturally introducing flexibility from the architecture. Looped models have been shown to have an inductive bias towards learning algorithms \citep{yang2024loopedtransformersbetterlearning, fan2025loopedtransformerslengthgeneralization} and have achieved remarkable success on reasoning benchmarks.  For example, Hierarchical Reasoning Model (HRM)~\citep{DBLP:journals/corr/abs-2506-21734} and Tiny Reasoning Model (TRM)~\citep{jolicoeurmartineau2025morerecursivereasoningtiny} incorporate a hierarchical structure into the looping process, proving effective in solving puzzle tasks such as sudoku, maze, and ARC-AGI~\citep{chollet2025arcprize2024technical}. 

\looseness=-1
The halting decision (i.e., deciding when to stop iterating), however, is no longer trivial in looped models. Most approaches either fix or randomly sample the number of loops~\citep{DBLP:journals/corr/abs-2602-11451, DBLP:journals/corr/abs-2510-25741, geiping2025scalingtesttimecomputelatent, saunshi2025reasoninglatentthoughtspower, DBLP:journals/corr/abs-2404-07839, DBLP:conf/nips/BansalSBEHGG22}, which eliminates adaptivity, or use external modules trained to make halting decisions~\citep{DBLP:journals/corr/abs-2506-21734, jolicoeurmartineau2025morerecursivereasoningtiny, dehghani2019universaltransformers}. The latter is associated with separate Adaptive Computation Time (ACT) networks~\citep{DBLP:journals/corr/Graves16}, which introduce optimization challenges, as they require a continuous relaxation of a discrete objective. As a result, ACT can fail to provide adaptivity, which we demonstrate in the case of  HRM and TRM (Figs.~\ref{fig:state-tracking-per-k},~\ref{fig:adaptivity_sudoku}%, more discussions in~\Cref{sec:act-adaptive}
). We mitigate this issue by introducing a different halting mechanism: let the model loop until its hidden state converges to a fixed-point, and use the convergence itself as the halting signal. Unlike ACT, the fixed-point halting mechanism requires no external module and lets the model spend as much compute as a given input demands.

\begin{figure*}[t] % use figure instead of figure* in a single-column document
    \centering

    % -------- Left: combined figure with two subfigures --------
    \begin{minipage}[t]{0.64\textwidth}
        \centering

        \begin{subfigure}[t]{0.48\linewidth}
            \centering
            \includegraphics[width=\linewidth]{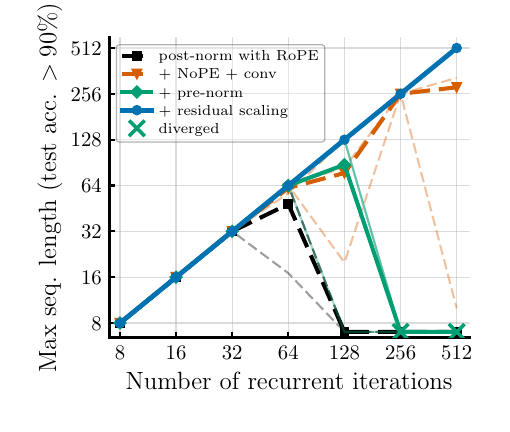}
            \caption{Trainability vs. context length}
            \label{fig:ut_max_len_vs_iters}
        \end{subfigure}
        \hfill
        \begin{subfigure}[t]{0.48\linewidth}
            \centering
            \includegraphics[width=\linewidth]{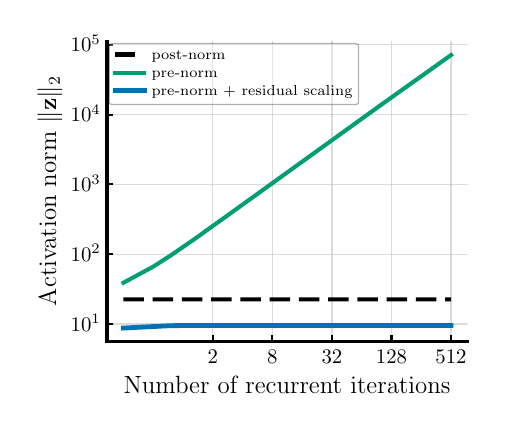}
            \caption{Activation norm at init.}
            \label{fig:ut_activation_norm}
        \end{subfigure}

        \caption{
            \textbf{The blessing and the curse of depth in looped Transformers.}
            Increasing the number of effective layers can unlock expressivity, but also creates a stability challenge: pre-norm models without residual scaling can diverge in activation norm, while post-norm models may struggle to utilize the signal.
        }
        \label{fig:ut_norm}
    \end{minipage}
    \hfill
    % -------- Right: separate figure --------
    \begin{minipage}[t]{0.32\textwidth}
        \centering
        \includegraphics[
            trim=230pt 437pt 874pt 110pt,
            clip,
            width=\linewidth
        ]{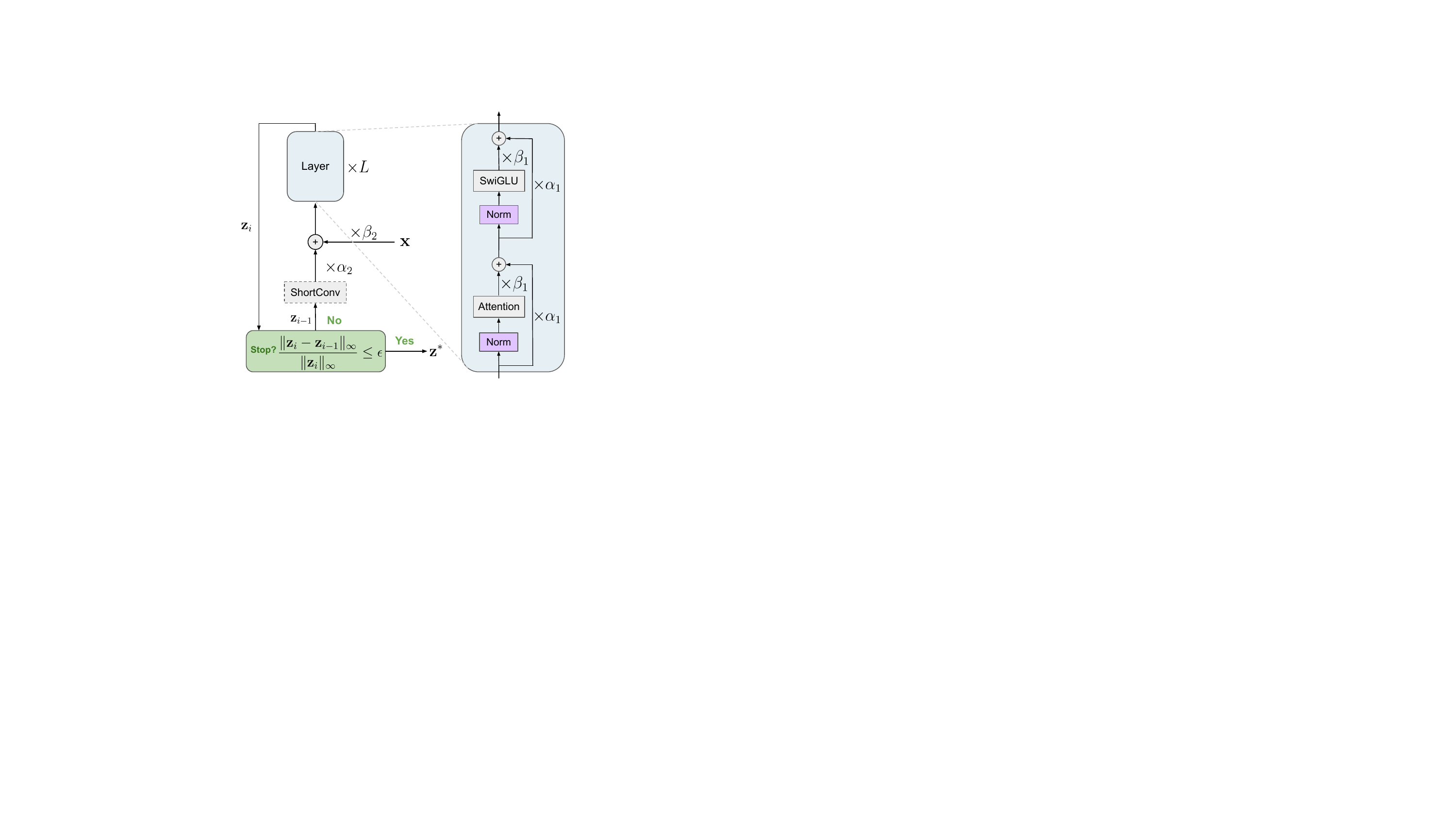}
        \vspace{0.51em}
        \caption{
            \textbf{FPRM architecture.}
            Our fixed-point looped Transformer uses pre-norm and residual
            scaling for improved signal propagation.
        }
        \label{fig:lofat-architecture}
    \end{minipage}

\end{figure*}

A further challenge arises as the number of loops in the model increases. This is desirable for harder tasks, but unrolling the recursions yields an effectively deep model (with a deep computation graph). As a result, looped architectures suffer from the same \textbf{signal propagation} issues as deep non-looped networks. To this end, we adapt architectural techniques for deep \efflayers in traditional Transformers to the looped Transformer setting. Our key intuition is that \textit{looped Transformers are, in part, very deep Transformers}.

Viewed this way, one design choice is surprising: looped models commonly use post-norm~\citep{dehghani2019universaltransformers, geiping2025scalingtesttimecomputelatent, DBLP:journals/corr/abs-2506-21734, jolicoeurmartineau2025morerecursivereasoningtiny}. Deep non-looped architectures, by contrast, prefer pre-norm, since post-norm causes unstable training~\citep{DBLP:conf/icml/XiongYHZZXZLWL20, DBLP:conf/nips/NociABOSL22}.
Yet in looped models, post-norm serves a specific purpose. It keeps the magnitude of activations bounded as the model iterates, preventing the hidden states from diverging~\citep{labovich2026stabilitygeneralizationloopedtransformers} (\Cref{fig:ut_norm}).
This raises the question: \textit{Can we switch from post-norm to pre-norm while keeping the activations bounded?} If so, this could make the training of looped models stable at large depths. In this paper, we use residual scaling parameters to give an affirmative answer to this question. In~\Cref{fig:fprm_trm_halting_points}, we observe that fixed-point halting, together with better signal propagation, allows FPRM, a \textbf{non}-hierarchical reasoning model, to outperform TRM at a lower cost on the Sudoku-Extreme benchmark. We provide a description of how the figure was made in~\Cref{sec:fig1-descript}. In~\Cref{sec:background}, we provide an overview of the related literature.

We summarize our contributions as follows:

\begin{enumerate}
    \item \textbf{Successfully training a looped pre-norm Transformer.} We modify a Transformer layer to be trainable over deeper \efflayers by switching post-norm to pre-norm and adding residual scaling parameters: one over each sub-layer of the network, and one across the looping iterations. 

    \item \textbf{Reaching a fixed-point as a halting mechanism.} To enable stable training, we propose a theoretically motivated modification specific to fixed-point models to limit the oscillation around the fixed-point.

    \item \textbf{Proposing the framework FPRM: Fixed-Point Reasoning Model.} We show that it outperforms the baselines on Sudoku-Extreme, Maze-Hard, ARC-AGI-1, and state-tracking benchmarks $A_5$ and $S_5$ among 7M parameter models. Notably, we achieve our results without the hierarchical structure of HRM and TRM. FPRM exhibits adaptivity of compute to task difficulty through a fixed-point halting mechanism (Figs.~\ref{fig:fprm_trm_halting_points},~\ref{fig:state-tracking-per-k},~\ref{fig:adaptivity_sudoku}).
\end{enumerate}

\section{Method}
\label{sec:method}
Looped architectures rely on bounded activations for stability: as the model loops, an unbounded layer can cause the hidden states to grow without limit (\Cref{fig:ut_norm}). Current methods mostly employ post-norm to satisfy the boundedness condition~\citep{jolicoeurmartineau2025morerecursivereasoningtiny, DBLP:journals/corr/abs-2506-21734, geiping2025scalingtesttimecomputelatent}. However, in fixed-depth models post-norm introduces a signal propagation issue, often associated with unstable training and restricted \efflayers, as reported by~\citet{DBLP:conf/icml/DongCL21, DBLP:conf/nips/NociABOSL22, DBLP:journals/corr/abs-2502-05795}. In~\Cref{sec:pre_post_toy}, we provide evidence for trainability issues in a toy looped model with post-norm, despite stability induced by bounded activations.
 
Consequently, in this section we propose to \textbf{(a) switch to pre-norm} to recover trainability at higher depth, while preserving boundedness by \textbf{(b) scaling the residual stream and the sub-layer outputs} (the attention and feed-forward maps). Together, these modifications yield a layer that is both stable under looping and trainable at large depth. Additionally, we introduce \textbf{(c) a fixed-point halting mechanism} that decides \efflayers adaptively per input. Interestingly, we find that with these changes we can \textbf{(d) remove the hierarchy} common to recent looped Transformers~\citep{DBLP:journals/corr/abs-2506-21734, jolicoeurmartineau2025morerecursivereasoningtiny}, resulting in a much simpler model. We provide an overview of our proposed model in~\Cref{fig:lofat-architecture}.
\subsection{Improving signal propagation with pre-norm} \label{sec: sig_prop}
We start by introducing pre-norm and post-norm in a Transformer layer. A Transformer layer consists of two sub-layers (with $f_{\theta^\ell}^\ell(.)$ denoting the $\ell^{th}$ sub-layer, where $\ell\in\{1, \dots, 2L\}$): multi-head attention and a feed-forward network. We denote the looped Transformer model as defined in~\Cref{eqn:looping}, consisting of multiple Transformer layers, by $f_\theta(.)$. In the Transformer layer, the two sub-layers are interleaved with layer normalization ($\mathrm{Norm}$). Two canonical placements of the normalization define two variants of the layer. The original \emph{post-norm} formulation ($\mathrm{Norm}_{\mathrm{post}}$)~\citep{DBLP:conf/nips/VaswaniSPUJGKP17} applies normalization after the residual addition with the residual stream ($\bmz^{\ell-1}$), while the \emph{pre-norm} variant ($\mathrm{Norm}_{\mathrm{pre}}$)~\citep{DBLP:conf/icml/XiongYHZZXZLWL20} applies normalization to the input of each sub-layer without modifying the residual stream:
\begin{equation*}
    \bmz^\ell = \mathrm{Norm}_{\mathrm{post}}\!\left(\bmz^{\ell-1} + f^\ell_{\theta^\ell}\!\left(\mathrm{Norm}_{\mathrm{pre}}\!\left(\bmz^{\ell-1}\right)\right)\right), \qquad \ell = 1, \ldots, 2L,
\end{equation*}
where $L$ is the number of layers (Transformer blocks). In fixed-depth models, both normalization placements have been linked to training issues at large depth: post-norm bounds the activation magnitudes but induces a signal propagation problem~\citep{DBLP:conf/icml/Kim0POKYSHSY25, DBLP:conf/nips/NociABOSL22}, while pre-norm improves signal propagation but causes growth in the residual magnitude~\citep{DBLP:conf/icml/Kim0POKYSHSY25, DBLP:conf/icml/XiongYHZZXZLWL20}. Therefore, while using pre-norm in a deep neural network is desirable from the signal propagation standpoint, it can introduce instability due to unbounded activations. 

Motivated by these observations, we investigate both normalization placements in Looped Transformers~\citep{saunshi2025reasoninglatentthoughtspower, dehghani2019universaltransformers}. In~\Cref{fig:ut_max_len_vs_iters}, we test the positive correlation between \efflayers and the expressivity (maximum sequence length with $>\!90\%$ test accuracy) of a looped Transformer for the state-tracking task $A_5$~\citep{merrill2025illusionstatestatespacemodels}. We observe that increasing the number of \efflayers of a looped Transformer with post-norm does not translate into improved expressivity, while a pre-norm variant diverges at larger depth. This divergence can be attributed to the unbounded growth of the activations, apparent in~\Cref{fig:ut_activation_norm}.
Therefore, to use pre-norm and improve signal propagation in deep looped models, we must first stabilize it, which we do below.

\subsection{Recovering boundedness via residual scaling}
\label{sec:bounded}

While pre-norm mitigates trainability problems in looped architectures, it removes the necessary boundedness condition that motivated the use of post-norm in them. This effect can be observed in~\Cref{fig:ut_activation_norm}, where the activations of the pre-norm model grow with deeper \efflayers, causing trainability issues observable in~\Cref{fig:ut_max_len_vs_iters}. Consequently, we propose to restore boundedness by introducing scaling parameters applied at two different scales: one over each sub-layer of the network, and one across the iterations.

\textbf{Layer-wise residual scaling.} Within a single application of $f_\theta\left(\bmz; \bmx\right)$, the residual stream and sub-layer output $f_{\theta^\ell}^\ell(\bmz^{\ell-1})$ are weighted by tied scalars $(\alpha_1, \beta_1)$ shared across all $L$ layers:
\begin{lightgraybox}
    \begin{equation}
        \bmz^\ell = \alpha_1\, \bmz^{\ell-1} + \beta_1\, f^\ell_{\theta^\ell}\!\left(\mathrm{Norm}_{\mathrm{pre}}\!\left(\bmz^{\ell-1}\right)\right), \qquad \ell = 1, \ldots, 2L.
        \label{eq:layer_scaling}
    \end{equation}
\end{lightgraybox}
\textbf{Iteration-wise input mixing.} Between consecutive applications of $f_\theta\left(\bmz;\bmx\right)$, we re-inject the input $\bmx$ with tied per-channel scalars $(\alpha_2, \beta_2)$ shared across all iterations~\citep{DBLP:conf/nips/BaiKK19}:
\begin{lightgraybox}
    \begin{equation}
        \bmz_{i+1}^0 = \alpha_2\, \bmz_i^{2L} + \beta_2\, \bmx.
        \label{eq:iter_mixing}
    \end{equation}
\end{lightgraybox}
The two scaling schemes are not independent. With an appropriate coupling between them, the resulting recurrence is \textit{bounded for any input}, resulting in stable looping~\citep{DBLP:conf/icml/OrvietoSGFGPD23}. In the following statement, we formalize this claim: 
\begin{lightgrayframe}    
    \begin{theorem}[Boundedness of FPRM iterates]
    \label{thm:bounded-fp-iterates}
    Consider the model defined by Equations~\ref{eq:layer_scaling} and~\ref{eq:iter_mixing}, and assume each layer map satisfies
    $\Vert f^\ell_{\theta^\ell}(\bmu)\Vert \leq c_f$ for all $\ell$ and all inputs $\bmu$. Let $0\leq \alpha_1, \alpha_2 < 1$, and set
    \begin{equation*}
        \beta_2 = 1 - \alpha_2\alpha_1^{2L},
        \qquad
        \beta_1 = \frac{\beta_2\,(1-\alpha_1)}{1-\alpha_1^{2L}}.
    \end{equation*}
    Then the fixed-point iterates $\{\bmz_i^0\}_{i\geq 0}$ from~\Cref{eq:iter_mixing} are bounded, and if $\bmz_i^0 \to \bmz_\infty^0$, then
    \begin{equation*}
        \Vert \bmz_\infty^0 \Vert \leq \Vert \bmx \Vert + \alpha_2\, c_f.
    \end{equation*}
    \end{theorem}
    \noindent The proof is in~\Cref{app:alphas-proof}. 
\end{lightgrayframe}

Note that the boundedness condition of the sequence model in~\Cref{thm:bounded-fp-iterates} is satisfied when using pre-norm, as shown by~\citet{DBLP:conf/icml/KimPM21}. However, boundedness still does not guarantee that looping will converge to a fixed-point, which we propose to utilize for adaptivity. In the following, we show that there exists a choice of $\alpha_2$ that satisfies this requirement. Consequently, as empirically shown by~\citet{DBLP:conf/nips/BansalSBEHGG22}, the model may become locally contractive during training. Note that Theorems~\ref{thm:damped-fp-convergence} and~\ref{thm:damping-stability} specialize classical fixed-point results to our setting and motivate our design choices. However, we do not claim these results as novel.

\begin{theorem}[Small $\alpha_2$ implies convergence]
\label{thm:damped-fp-convergence}
Let $\lambda_f$ be the Lipschitz constant of the $L$-layer model defined by
\Cref{eq:layer_scaling}, i.e.\ the map $\bmz^0 \mapsto \bmz^{2L}$. Then the
looped step $f_\theta(\,\cdot\,;\bmx)$ of Eqs.~\ref{eq:layer_scaling}--\ref{eq:iter_mixing}
is Lipschitz in its first argument with constant $\alpha_2\lambda_f$. In particular,
if $\alpha_2\lambda_f < 1$, then $f_\theta(\,\cdot\,;\bmx)$ is a contraction and the
iteration $\bmz_{i+1} = f_\theta(\bmz_i;\bmx)$ converges to a unique fixed-point
$\bmz^\star=f_\theta(\bmz^*;\bmx)$ at a linear rate:
\begin{equation*}
    \left\Vert f_\theta(\bmz_i;\bmx) - \bmz_i \right\Vert
    \leq
    \left(\alpha_2\lambda_f\right)^i
    \left\Vert f_\theta(\bmz_0;\bmx) - \bmz_0 \right\Vert.
\end{equation*}
\end{theorem}
\noindent The proof is in~\Cref{app:damped-fp-convergence-proof}.

While a contractive map guarantees convergence, an excessively contractive one severely limits expressivity~\citep{DBLP:conf/nips/BaiKK19, DBLP:conf/nips/AnilPLTWBKG22}. We avoid this by making $\alpha_1$ and $\alpha_2$ learnable. In practice, we find that initializing the network to be more contractive by setting $\alpha_2$ to be small yields better performance (\Cref{tab:alpha_sweep}). In~\Cref{fig:alpha_distribution}, we observe that after training, the distribution of $\alpha$ values widens, with the median very close to the initial point. Intriguingly, these observations are also in line with the common solutions to signal propagation and rank-collapse problems in deep neural networks~\citep{DBLP:conf/nips/NociABOSL22, DBLP:journals/corr/abs-2502-05795}, suggesting a connection between rank-collapse and the signal propagation issue in looped architectures. Together, these modifications yield a pre-norm looped Transformer that maintains performance over longer looping horizons before saturating, compared to the post-norm variant (\Cref{fig:ut_norm}).

\subsection{Oscillation around the fixed-point}
\label{sec:damping-alg}

So far, we have been able to establish that, given a small enough $\alpha_2$, the model introduced in~\Cref{eq:iter_mixing} becomes locally contractive and converges to a fixed-point. However, in practice, the contraction factor of~\Cref{thm:damped-fp-convergence} is not guaranteed. For some inputs, we observe that the model often exhibits oscillatory behavior, causing the iteration to stay in a small region of latent space without converging. This non-convergent behavior is not in tension with~\Cref{thm:damped-fp-convergence}, as the theorem gives a sufficient condition for convergence, not a complete characterization of the iteration's behavior. In fact, oscillation around the fixed-point can happen when the Jacobian satisfies certain conditions. 

Linearizing the iteration near a fixed-point $\bmz^\star$ gives $\bmz_{i+1} - \bmz^\star \approx \matJ\,(\bmz_i - \bmz^\star)$, where $\matJ = \partial f_\theta / \partial \bmz \,|_{\bmz^\star}$. Oscillation around $\bmz^\star$ can arise when $\matJ$ has an eigenvalue with $\Re(\lambda_i) < 1$ but $|\lambda_i| \geq 1$, in which case the iteration spirals around $\bmz^\star$ rather than contracting toward it. The half-plane condition $\Re(\lambda_i) < 1$ is exactly what licenses a runtime fix that does not require modifying $f_\theta$:

\begin{lightgrayframe}
    \begin{theorem}[Damping stabilizes oscillatory fixed-point dynamics]
    \label{thm:damping-stability}
    Suppose $f_\theta(\,\cdot\,;\bmx)$ is continuously differentiable in a neighborhood of a fixed-point $\bmz^\star$, and that every eigenvalue $\lambda_i$ of the Jacobian $\matJ$ at $\bmz^\star$ satisfies $\Re(\lambda_i) < 1$. Define the damped iteration map
    \begin{equation*}
        g_{\eta,\theta}(\bmz; \bmx) \;:=\; \eta\, f_\theta(\bmz;\bmx) + (1-\eta)\, \bmz.
    \end{equation*}
    Then there exists $\eta_0 \in (0, 1)$ such that, for every $\eta \in (0, \eta_0)$, the iteration $\bmz_{i+1} = g_{\eta,\theta}(\bmz_i;\bmx)$ converges locally to $\bmz^\star$. Moreover, $g_{\eta,\theta}(\,\cdot\,;\bmx)$ and $f_\theta(\,\cdot\,;\bmx)$ have the same fixed-points.
    \end{theorem}
    \noindent The proof is in~\Cref{app:damping-proof}.
\end{lightgrayframe}

\Cref{thm:damping-stability} shows that a suitable damping factor $\eta$ eliminates the oscillations while preserving the fixed-points of $f_\theta(\bmz;\bmx)$. We measure convergence to a fixed-point at iteration $i$ as
\begin{equation*}
    r_i \;=\; \frac{\Vert \bmz_i - f_\theta(\bmz_i;\bmx) \Vert_\infty}{\Vert f_\theta(\bmz_i;\bmx) \Vert_\infty + \epsilon},
\end{equation*}
which serves as the halting signal: the iteration stops once $r_i$ falls below a tolerance $\tau$. To choose $\eta$ adaptively at inference time, we use a patience mechanism that decreases $\eta$ whenever this residual stops improving. We track the smallest residual observed so far, $r^\star = \min_{j \le i} r_j$. A geometric decay $\eta \gets \gamma\, \eta$ with $\gamma \in (0, 1)$ is applied to the step-size after $P$ consecutive iterations with no improvement in the residuals. The full procedure is given in~\Cref{alg:fp_optimizer}, which is based on the implementation provided by~\citet{movahedi2025fixedpointrnnsinterpolatingdiagonal}. In~\Cref{sec:fp-optimizer}, we investigate the impact of the proposed modifications on the gradient-based optimization of a looped Transformer.

\subsection{The fixed-point reasoning model}

So far, we have introduced modifications to improve signal propagation in Looped Transformers (\Cref{sec: sig_prop}) without sacrificing stability (\Cref{sec:bounded}). These modifications yield fixed-point iterations that can be made non-oscillatory (\Cref{sec:damping-alg}) and trainable through truncated-BPTT (\Cref{sec:fp-optimizer}), making adaptivity through fixed-points reliable. We assemble these components into FPRM, summarized in~\Cref{alg:fprm_training} and illustrated in~\Cref{fig:lofat-architecture}. The result is a looped Transformer that iterates until its hidden state converges, spending compute proportional to each input's difficulty---the adaptivity our halting mechanism was designed to provide. Since we did not observe improvements in our experiments using the hierarchical structure of~\citet{DBLP:journals/corr/abs-2506-21734}, we opt for a classic looped architecture~\citep{dehghani2019universaltransformers}. An overview of the framework is available in~\Cref{alg:fprm_training}. In~\Cref{architecture_details} we provide further details on FPRM. In~\Cref{app:fprm-ablation}, we provide an ablation on the architecture choices in FPRM.

\begin{figure}[t]
\begin{tcbraster}[raster columns=2, raster equal height, raster column skip=0.04\textwidth]
\begin{tcolorbox}[algbox]
\algcaption{FPRM training loop with truncated BPTT and deep supervision}\label{alg:fprm_training}
\begin{algorithmic}[1]
\Require Model $f_\theta$; prediction head $h_\phi$; fixed-point optimizer \textsc{FPOpt}; model optimizer \textsc{ModelOpt}; input $\bmx$; target $\bmy$; BPTT depth $K$; initial state $\bmz_0$
\State $\bmz \gets \bmz_0$
\While{\textsc{FPOpt}.\textsc{cont}()} \Comment{outer loop}
    \For{$k = 1, \ldots, K$} \Comment{BPTT window}
        \State $\tilde\bmz_k \gets f_\theta(\bmz;\, \bmx)$
        \State $\bmz \gets \textsc{FPOpt}.\textsc{step}(\bmz,\, \tilde\bmz_k)$
    \EndFor
    \State $\hat{\bmy} \gets h_\phi(\bmz)$ \Comment{deep supervision}
    \State $\mathcal{L} \gets \textsc{CrossEntropy}(\hat{\bmy},\, \bmy)$
    \State \textsc{ModelOpt}.\textsc{backward}($\mathcal{L}$)
    \State $\bmz \gets \mathrm{detach}(\bmz)$
\EndWhile
\end{algorithmic}
\end{tcolorbox}
\begin{tcolorbox}[algbox]
\algcaption{Fixed-point optimizer \textsc{FPOpt}: one damped step with patience-based decay}\label{alg:fp_optimizer}
\begin{algorithmic}[1]
\Require initial damping $\eta_0$; decay $\gamma \in (0,1)$; patience $P$
\State \textbf{Internal state:} $\eta \gets \eta_0$,\ \ $p \gets P$,\ \ $r^\star \gets \infty$
\Statex
\Procedure{Step}{$\bmz,\, \tilde{\bmz}$}
    \State $r \gets \|\bmz - \tilde{\bmz}\|_\infty / (\|\tilde{\bmz}\|_\infty + \epsilon)$ \Comment{residual}
    \State $\bmz \gets \eta\, \tilde{\bmz} + (1 - \eta)\, \bmz$ \Comment{damped update}
    \If{$r < r^\star$}
        \State $r^\star \gets r$,\ \ $p \gets P$ \Comment{progress: reset}
    \Else
        \State $p \gets p - 1$
        \If{$p \leq 0$ \textbf{and} $r > \tau$}
            \State $\eta \gets \gamma\, \eta$,\ \ $p \gets P$ \Comment{decay $\eta$}
        \EndIf
    \EndIf
    \State \Return $\bmz,\, r$
\EndProcedure
\end{algorithmic}
\end{tcolorbox}
\end{tcbraster}
\vspace{-15pt}
\end{figure}

During training, the model loops in windows of $k$ iterations, with $k$ being a hyperparameter that determines the truncated-BPTT value. During inference, we set $k=1$. At each forward pass, the fixed-point optimizer introduced in~\Cref{alg:fp_optimizer} is called to dampen the fixed-point iteration steps. Then, we get a prediction from the current state $\bmz$ of the model and perform a deep supervision step, following~\citet{DBLP:journals/corr/abs-2506-21734, jolicoeurmartineau2025morerecursivereasoningtiny} (\Cref{architecture_details}). To truncate the computation graph between deep supervision steps during training, we detach the state $\bmz$ from the graph and stop when the optimizer detects fixed-points. This happens when the residual falls below the tolerance or the step-size becomes too small.

\section{Experiments}
\label{sec:exps}
We evaluate FPRM against looped reasoning models on puzzle, adaptivity, and signal-propagation benchmarks. Our implementation builds on the public TRM codebase~\citep{jolicoeurmartineau2025morerecursivereasoningtiny} and adopts its deep-supervision training procedure. Additional experimental details are provided in ~\Cref{experimental details}. In~\Cref{sec:data-description} we describe the datasets. In our experiments, we evaluate FPRM's puzzle-solving performance, adaptivity to task difficulty, and depth-induced signal-propagation effects against baselines. In~\Cref{app:variance}, we compare FPRM and TRM across multiple training seeds.

\begin{table*}[ht]
    \centering
    \caption{
        Test accuracy on Sudoku-Extreme, Maze-Hard, ARC-AGI-1, and ARC-AGI-2. For each task, the \textbf{best overall} result is bold-face, and the \underline{best result with 7M parameters} is underlined. The results denoted by~$^\dagger$ are reproduced using public \href{https://huggingface.co/arcprize/trm_arc_prize_verification}{checkpoints}. The ARC results for URM, denoted by ~$^\ddagger$, are only reported for Pass@1.
    }
    \vspace{-1mm}
    \small
    \setlength{\tabcolsep}{6pt} % default is usually 6pt
    \renewcommand{\arraystretch}{1.05}
    \begin{tabular}{llccccc}
        \toprule
        \textbf{Model} & \textbf{\# params.}  & \textbf{Single Loop}& \textbf{Sudoku-Ext.} & \textbf{Maze-Hard} & \textbf{ARC-1} & \textbf{ARC-2} \\
        & & \textbf{(No Hier.)} & \textbf{Pass@1} & \textbf{Pass@1} & \textbf{Pass@2} & \textbf{Pass@2}
        \\
        \midrule
        \multicolumn{7}{l}{\textcolor{textgray}{\textit{Our reproduction attempt}}} \\
        \midrule
        \textcolor{textgray}{Attractor Model} & \textcolor{textgray}{27M}& \textcolor{textgray}{\xmark} & \textcolor{textgray}{71.4\%} & \textcolor{textgray}{--} & \textcolor{textgray}{--} & \textcolor{textgray}{--} \\
        \textcolor{textgray}{TRM} & \textcolor{textgray}{7M} & \textcolor{textgray}{\xmark} & \textcolor{textgray}{72.6\%} & \textcolor{textgray}{79.0\%} & \textcolor{textgray}{40.0\%$^\dagger$} & \textcolor{textgray}{6.2\%$^\dagger$}
        \\ \midrule
        \textit{Reported} \\
        \midrule
        HRM  & 27M & \xmark & 55.0\% & 74.5\% & 40.3\% & 5.0\% \\
        Attractor Model & 27M & \xmark & 91.4\% & \textbf{93.1\%} & -- & -- \\
        URM & 14M & \xmark & 77.6\% & -- & \textbf{$\geq$53.8\%}$^\ddagger$ & \textbf{$\geq$16.0\%}$^\ddagger$ \\
        TRM  & 7M & \xmark & 74.7\% & 85.3\% & 44.6\% & \underline{7.8}\% \\
        EqR & 7M & \xmark & 93.0\% & -- & -- & -- \\
        Attractor Model & 7M & \xmark & 54.3\% & 46.7\% & -- & -- \\
        \rowcolor{fprmbluelight} FPRM & 7M & \cmark & \underline{\textbf{94.2\%}} & \underline{87.0\%} & \underline{47.5\%} & 6.2\% \\
        \bottomrule
    \end{tabular}

    \vspace{-10pt}
    \label{tab:sudoku-maze}
\end{table*}
%\subsection{Puzzles}
\subsection{Puzzle tasks}
\label{sec:puzzle-exps}
We first evaluate FPRM on puzzle problems, namely the Sudoku-Extreme, Maze-Hard~\citep{DBLP:journals/corr/abs-2506-21734}, ARC-1, and ARC-2~\citep{chollet2025arcprize2024technical} benchmarks. These benchmarks were designed to test whether latent recurrent reasoning models can solve symbolic search problems from limited samples. In the following, we discuss the experimental results. For more information on the baselines, we refer the reader to~\Cref{sec:background}.

% \subsubsection{Comparison between baselines and FPRM}
\Cref{tab:sudoku-maze} demonstrates the effectiveness of FPRM as the best performing model with 7M parameters on Sudoku-Extreme, Maze-Hard, and ARC-1, with performance on par with TRM on ARC-2. On Sudoku-Extreme, FPRM improves upon even larger models. It is worth noting that these results are achieved without the breadth-search fixed-point method proposed by~\citet{huang2026equilibriumreasonerslearningattractors}, which is orthogonal to FPRM's modifications. On Maze-Hard, FPRM underperforms compared to the larger Attractor model with \textasciitilde$\times\!4$ more parameters. However, we note that we were not able to reproduce the results reported in the paper~\citet{feinashley2026solveloopattractormodels}. 

On ARC-2, FPRM performs similarly to the publicly available TRM \href{https://huggingface.co/arcprize/trm_arc_prize_verification}{checkpoints}, but underperforms compared to the TRM results reported by~\citet{jolicoeurmartineau2025morerecursivereasoningtiny} and the URM baseline, which has roughly twice as many parameters. However, considering the recent trend, we emphasize that the ARC benchmarks appear to be much more sensitive to parameter count compared to the other puzzle tasks considered in this section~\citep{DBLP:journals/corr/abs-2602-02156, DBLP:journals/corr/abs-2511-14761}. Therefore, we note that the comparison may not be fair in this case. 

Given that FPRM performs on par with or better than the hierarchical baselines that use post-norm on the tasks from~\Cref{tab:sudoku-maze}, we hypothesize that the hierarchy might alleviate signal propagation issues in these models. However, if this hypothesis holds, the benefit of the hierarchy appears limited: in \Cref{fig:sudoku_inference_scaling}, both models improve as \efflayers grow, but the performance gap widens in favor of FPRM, and TRM's performance stays below FPRM at matched \efflayers.

\subsection{Adaptivity}
\label{sec:adaptity}

In this section, we demonstrate the adaptivity of FPRM, and compare it to TRM. We use three benchmarks: Sudoku-Extreme from~\Cref{sec:puzzle-exps} and state-tracking benchmarks $A_5$ and $S_5$ introduced in~\cite{merrill2025illusionstatestatespacemodels}. Each benchmark provides an intuitive measure of difficulty. For Sudoku-Extreme, this is the number of empty cells, an established proxy for difficulty~\citep{prates2018problemsolvingedgechaos}. For state-tracking, it is the sequence length.

\begin{figure*}[t]
  \centering

  \includegraphics[width=0.92\textwidth]{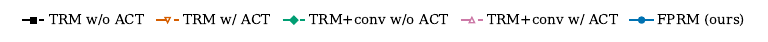}

  \vspace{0.35em}

  \setlength{\tabcolsep}{0pt}%
  \begin{subfigure}[t]{0.245\textwidth}
    \centering
    \includegraphics[width=\linewidth]{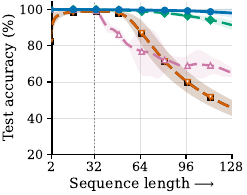}
    \caption{Accuracy, $A_5$}
  \end{subfigure}\hfill
  \begin{subfigure}[t]{0.245\textwidth}
    \centering
    \includegraphics[width=\linewidth]{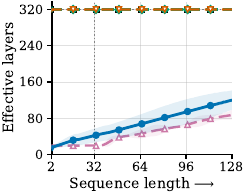}
    \caption{Adaptivity, $A_5$}
  \end{subfigure}\hfill
  \begin{subfigure}[t]{0.245\textwidth}
    \centering
    \includegraphics[width=\linewidth]{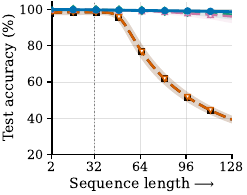}
    \caption{Accuracy, $S_5$}
  \end{subfigure}\hfill
  \begin{subfigure}[t]{0.245\textwidth}
    \centering
    \includegraphics[width=\linewidth]{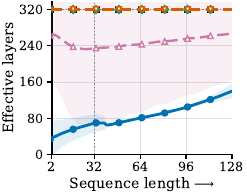}
    \caption{Adaptivity, $S_5$}
  \end{subfigure}

  \caption{\textbf{Length generalization and adaptive compute as a function of sequence length.}
  Shaded bands show 95\% confidence intervals over seeds. The vertical dotted line marks
  the training length 32. The matched compute budget is 320 effective layers.}
  \label{fig:state-tracking-per-k}
  \vspace{-3pt}
\end{figure*}

% \vspace{-1mm}

\textbf{State tracking.}
As shown in~\Cref{fig:state-tracking-per-k}, plain TRM fails to length-extrapolate: at length 128, both TRM and TRM with enabled ACT obtain $45.8\%\pm3.9\%$ on $A_5$ and $39.4\%\pm1.9\%$ on $S_5$. Adding a short causal 1D convolution layer substantially improves length generalization, reaching $91.4\%\pm2.3\%$ on $A_5$ and $97.2\%\pm2.5\%$ on $S_5$. This modification is not part of the original TRM, but matches the task structure: group composition is a local left-to-right scan, $s_i=s_{i-1}\cdot g_i$, where $g_i$ is the input permutation at step $i$, $s_i$ is the accumulated state, and $\cdot$ denotes permutation composition. Causal convolution provides a shared translation-equivariant primitive for this operation.

However, 1D convolution does not completely fix ACT. On $A_5$, TRM+conv+ACT scales compute with sequence length but drops to $65.3\%\pm2.2\%$, well below TRM+conv without ACT. On $S_5$, it reaches $96.2\%\pm3.2\%$, but uses substantially more compute than FPRM; moreover, only a few seeds learn to adapt, while the others exhaust the full compute budget. In contrast, FPRM scales compute smoothly with sequence length while maintaining high accuracy, reaching $98.1\%\pm2.2\%$ on $A_5$ and $98.8\%\pm0.9\%$ on $S_5$ at length 128.

\textbf{Sudoku-Extreme.} 
We compare the performance and efficiency of the halting mechanisms in FPRM and TRM on sudoku puzzles of varying difficulty, measured by the number of empty cells (\Cref{fig:adaptivity_sudoku}). The sample counts per difficulty level are shown in~\Cref{fig:fprm_empty_cells}. We observe that halting mechanisms in both TRM and FPRM adapt the number of loops, i.e., inference compute, to the sample difficulty. In contrast, the default behavior of TRM during inference (deactivated ACT) does not adapt the compute to the task difficulty. When we enable ACT with TRM, FPRM is more efficient and accurate.

\textbf{Compute--accuracy trade-off in fixed-point halting.}
Beyond adapting compute to task difficulty, FPRM's halting hyperparameters allow control over the inference compute--accuracy trade-off. This trade-off balances potential accuracy gains from additional iterations against their cost and is adjusted through the step-size decay rate $\gamma$ and maximum patience $P$ in~\Cref{alg:fp_optimizer}, which determine when the model halts. In our experiment from~\Cref{fig:fprm_patience_decay}, a larger $\gamma$ slows down the decay of $\eta$, allowing further latent representation updates and greater effective depth before halting. These additional iterations improve accuracy. Patience $P$ determines when a decay occurs, and its effect is minor, as for larger $\gamma$ the two patience curves coincide. We additionally evaluate damping during training and/or inference in~\Cref{app:damping-ablation}. The best mean accuracy is achieved with damping during both training and inference.

\subsection{Depth-induced signal propagation issues}
\label{sec:signal-prop}
We investigate the role of pre-norm with residual scaling in mitigating depth-induced signal propagation issues in looped models through three complementary analyses: \textbf{(1)} \textit{trainability}, where we show that pre-norm with residual scaling keeps activations bounded and enables stable training at large depth (\Cref{sec:boundedness}); \textbf{(2)} \textit{depth utilization}, examined below, where we show that boundedness alone is not sufficient, and that pre-norm with residual scaling additionally improves depth utilization in FPRM~\citep{DBLP:journals/corr/abs-2502-05795} (\Cref{sec:depth_util}); and \textbf{(3)} \textit{training dynamics of the residual scaling parameters}, where we analyze how the residual-scaling parameters evolve during training (\Cref{sec:alphas_analysis}).

\subsubsection{Pre-norm  with residual scaling enables higher depth utilization} \label{sec:depth_util}
An important measure for evaluating signal propagation issues in deep neural networks is depth utilization~\citep{DBLP:journals/corr/abs-2502-05795}. Depth utilization measures whether all layers contribute meaningfully to the model. It is commonly probed by removing layers from a trained model and measuring the effect on performance at test-time. In models with signal propagation problems, removing deeper layers (closer to the output) usually does not impact the performance significantly, highlighting that these models fail to utilize the compute due to signal propagation issues. However, in a looped model there is no fixed stack of layers to remove, since the same block is applied to the previous latent representations. Consequently, we cannot perform the same experiment here. We therefore probe the same property from the opposite direction: instead of \textit{removing computation} and measuring degradation, we \textit{add computation} and measure improvement. We do so in two complementary regimes. In the first, using the state-tracking task, we ask whether the model can be trained to use the depth that a harder task requires. In the second, using the Sudoku-Extreme task, we ask whether a trained model can convert depth beyond its training regime into further gains at test time.

\begin{figure*}[t] % use figure instead of figure* in a single-column document
    \centering

    % -------- Left: FPRM loop-utilization (shared legend + two subfigures) --------
    \begin{minipage}[t]{0.65\textwidth}
        \vspace{0pt} % <-- top-align this minipage with the one beside it
        \centering

        % Shared legend (spans both subfigures below)
        \includegraphics[width=\linewidth]{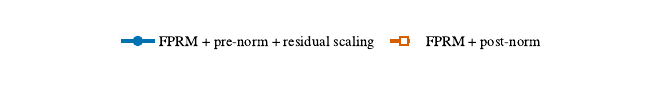}
        \vspace{-1.9em}

        \begin{subfigure}[t]{0.48\linewidth}
            \centering
            \includegraphics[width=\linewidth]{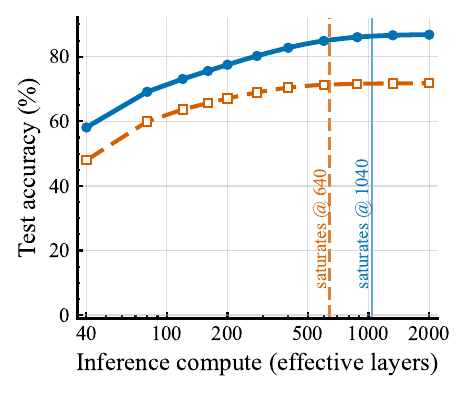}
            \caption{Test accuracy}
            \label{fig:fprm_norm_accuracy}
        \end{subfigure}
        \hfill
        \begin{subfigure}[t]{0.48\linewidth}
            \centering
            \includegraphics[width=\linewidth]{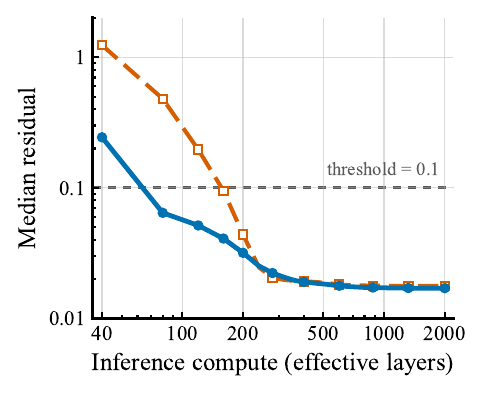}
            \caption{Fixed-point residuals}
            \label{fig:fprm_norm_residual}
        \end{subfigure}

        \caption{
            \textbf{Loop-utilization of FPRM on sudoku.} \textbf{(a)} Test accuracy of FPRM with pre-norm with residual scaling vs. post-norm. \textbf{(b)} Median residual. The pre-norm model utilizes the inference compute more efficiently, while both have similar residuals that stay below the convergence threshold.  %This indicates similar latent-space convergence, with more meaningful updates in the pre-norm variant, resulting in improved performance.
        }
        \label{fig:fprm_norm}
    \end{minipage} \hfill
    % -------- Right: stacked accuracy / compute vs. decay rate (separate legend) --------
    \begin{minipage}[t]{0.31\textwidth}
        \vspace{10pt} % <-- top-align with the left minipage (legends line up)
        \centering
        % Separate shared legend on top, then the stacked plot below
        \includegraphics[width=0.6\linewidth]{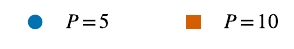}
        
        % \vspace{0.4em}
        \includegraphics[width=\linewidth]{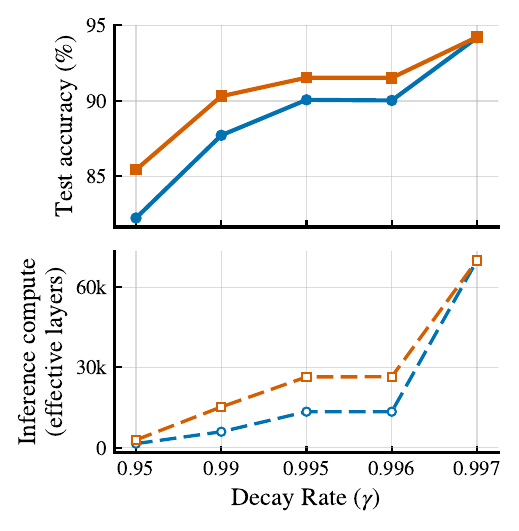}

        \vspace{-0.5em}
        
        \caption{
        \textbf{Decay rate and patience.}
    Accuracy (top) and effective depth (bottom)
    versus $\gamma$ for $P\in\{5,10\}$. 
            %\textbf{Decay rate and patience.} Test accuracy and \efflayer of FPRM with fixed-point halting as a function of decay rate $\gamma$, for maximum-patience $P\in\{5, 10\}$.
        }
        \label{fig:fprm_patience_decay}
    \end{minipage}
    \vspace{-10pt}
\end{figure*}

\textbf{State-tracking.} In~\Cref{sec:boundedness}, we discussed the necessity of using normalization layers that ensure the boundedness of the activations, which is required for stable training. However, bounded activations are not a sufficient condition to ensure effective utilization of large depth, once it is reached. We demonstrate this in~\Cref{fig:ut_max_len_vs_iters} using a state-tracking task, where increasing difficulty, corresponding to longer sequences, requires training at greater depths~\citep{merrill2025littledepthgoeslong, movahedi2025fixedpointrnnsinterpolatingdiagonal}. We train a looped Transformer model with its number of loops (\efflayers) tied to the training sequence length. We choose the sequence length from the set $\{8, 16, 32, 64, 128, 256, 512\}$. We plot the maximum sequence length solved with $>90\%$ accuracy against \efflayer. Because we also at test-time match \efflayer to the training sequence length, a model with no signal-propagation bottleneck should solve exactly the length its depth permits, tracing the identity line $y=x$ (Fig. 3 in~\citet{movahedi2025fixedpointrnnsinterpolatingdiagonal}). Only the model combining pre-norm with residual scaling consistently exhibits this behavior, providing evidence that these modifications improve trainability at large effective depths.

\textbf{Sudoku-Extreme.} Compared to the previous experiment on the state-tracking task, here we run each model far beyond its trained depth, trying to detect the point where more compute no longer translates into improvements at test-time. We expect the performance of a model with fewer signal propagation issues to saturate later and with more \efflayers. On the other hand, the performance of a model bottlenecked by signal propagation problems is expected to saturate early, indicating that it cannot convert the extra compute into better predictions. For this experiment, we focus on two variants of the FPRM model, one with pre-norm and residual scaling, and the other with post-norm, both trained on the Sudoku-Extreme task.

\looseness=-1
In~\Cref{fig:fprm_norm_accuracy}, we demonstrate that scaling test-time compute improves performance for both types of normalization. Furthermore, considering~\Cref{fig:fprm_norm_residual}, we also observe that the majority of the improvement comes before at least half of the samples halt. However, there are clear differences between the \efflayers of the two normalization methods in~\Cref{fig:fprm_norm_accuracy}, as the pre-norm model's performance saturates at almost twice as much compute, indicating improved signal propagation through depth. 

\looseness=-1
The advantage of pre-norm with residual scaling also extends to the cross-model comparison in~\Cref{fig:sudoku_inference_scaling}, between FPRM (pre-norm with residual scaling) and TRM (post-norm). In this inference-time scaling experiment, all samples are run to a varied maximum looping compute budget. For TRM, compute is scaled through deep supervision steps, which we find optimal relative to scaling L- and H-steps (\Cref{sec:trm_scaling_LH_vs_deep_sup}). FPRM outperforms TRM across a range of effective-layer depths reached, with the gap widening at higher compute budgets, consistent with FPRM making better use of its depth.

Moreover, in~\Cref{app:trm-sig-prop}, we introduce our architectural modifications to the Transformer layers used in TRM in an attempt to investigate whether they improve the performance once the hierarchy is present. We conclude that replacing post-norm with pre-norm and residual scaling helps in the case of higher~\efflayers in TRM on the Sudoku-Extreme task. We hypothesize that the improvement can partly be attributed to alleviating signal propagation issues.

\section{Discussion}\label{sec:discussion}
Our experiments support three broader observations about FPRM and looped reasoning models in general, which we discuss in turn.

\textbf{Looped fixed-point models are adaptive.} FPRM adapts to the difficulty of the problem more effectively than TRM (Figs.~\ref{fig:state-tracking-per-k},~\ref{fig:adaptivity_sudoku} and~\Cref{sec:adaptity}), using fewer \efflayers (compute), while achieving better performance. This is a consequence of FPRM halting closer to the saturation point of accuracy (\Cref{fig:fprm_scaling_convergence}). In contrast, TRM with its ACT halting mechanism either halts too early, resulting in lower performance, or too late, using excessive compute.

\textbf{Enabling ACT at inference time.} The originally proposed TRM does not use its trained ACT head at inference time, leading to non-adaptive behavior. However, we find this to be an engineering challenge rather than a fundamental limitation. Therefore, in Figs.~\ref{fig:fprm_trm_halting_points},~\ref{fig:state-tracking-per-k},~\ref{fig:adaptivity_sudoku}, we record the number of effective layers reached at the moment when the probability of halting exceeds $0.5$. Similarly, in the case of FPRM, we do so when the residual drops below the set threshold ($0.1$ in this case). However, note that the halted samples remain in the batch until the last sample in the batch halts. We leave the efficient implementation for future work.

\textbf{The role of hierarchy in HRM and TRM.} While originally hierarchical reasoning was biologically motivated~\citep{DBLP:journals/corr/abs-2506-21734}, later explanations involved likening the lower level of the hierarchy to a scratch pad, the latent representation of which is used by the higher level for prediction~\citep{jolicoeurmartineau2025morerecursivereasoningtiny}. However, the role of hierarchy as the driving force behind the success of hierarchical reasoning models has been brought into question recently, with non-hierarchical architectures performing as well as hierarchical ones~\citep{ge2025hierarchicalreasoningmodelsperspectives, arcprize2025hiddendrivers}. In~\Cref{sec:signal-prop}, we observe that a Transformer with post-norm, the building block of TRM and HRM, suffers from signal propagation issues. In contrast, \Cref{sec:puzzle-exps} shows that FPRM improves upon these models without a hierarchy by improving signal propagation. Consistently, reallocating TRM's compute from H- and L-steps to additional deep supervision steps improves its performance (\Cref{fig:trm_deep_superv_vs_LH}). Since fewer inner steps mean fewer \efflayers per supervision step, this modification reduces the post-norm signal-propagation issue identified in \Cref{sec:signal-prop}. We therefore hypothesize a simpler explanation for the success of hierarchical models: \textit{the hierarchy improves signal propagation}. The plausibility of this claim comes from the fact that the hierarchy introduces a skip-connection between the output of the previous L-step and the input of the next L-step in the form of input mixing for the H-step (Fig. 4 in~\citet{DBLP:journals/corr/abs-2506-21734}). However, as demonstrated in~\Cref{app:trm-sig-prop}, the hierarchy as an extra skip-connection introduces the same unbounded activation issues present in pre-norm without residual scaling. A theoretical account of this role is an interesting direction for future work.

\textbf{Scaling behavior of FPRM.} Our results (\Cref{sec:exps}) show that better signal propagation lets FPRM use compute more efficiently (\Cref{fig:sudoku_inference_scaling}), and since harder problems require more compute~\citep{merrill2025illusionstatestatespacemodels, movahedi2025fixedpointrnnsinterpolatingdiagonal}, this advantage is most visible on harder tasks (Figs.~\ref{fig:state-tracking-per-k},~\ref{fig:accuracy_difficulty_sudoku}). Moreover, since halting is governed by the fixed-point optimizer rather than a learned module, the decay rate $\gamma$ and patience $P$ provide a natural knob for selecting a point on the compute-performance Pareto front. However, looped models are not guaranteed to learn optimal algorithms. For example, $A_5$ admits a logarithmic depth solution, yet CoT requires a super-logarithmic number of steps~\citep{merrill2024expressivepowertransformerschain}. Similarly suboptimal scaling has been reported for recurrent-in-depth SSMs~\citep{movahedi2025fixedpointrnnsinterpolatingdiagonal} and is also visible in~\Cref{fig:state-tracking-per-k}. Finding a latent reasoning architecture that solves state tracking with logarithmic complexity while remaining Turing-complete~\citep{dehghani2019universaltransformers} remains an open challenge.

\textbf{Limitations.} In a similar spirit to previous literature on end-to-end reasoning  \citep{kaiser2016neuralgpuslearnalgorithms,fan2025loopedtransformerslengthgeneralization, DBLP:journals/corr/abs-2506-21734, jolicoeurmartineau2025morerecursivereasoningtiny, du2022learningiterativereasoningenergy, du2024learningiterativereasoningenergy}, we test our model only on algorithmic tasks and not on natural language. It is an open challenge to demonstrate that the compositional reasoning behavior that latent models exhibit on algorithmic tasks translates to other domains. In addition, even though the base architecture of FPRM could adopt any model (e.g. a CNN, an MLP, or state-space models), we limit our experiments to Transformers.

\section{Conclusion}
We present architectural modifications for looped fixed-point Transformers that enable the use of pre-norm, improving the model's ability to exploit deeper \efflayers provided by looping. These modifications allow FPRM to outperform hierarchical baselines of similar size, such as HRM and TRM, on common symbolic reasoning benchmarks. We show that on state tracking and Sudoku-Extreme, FPRM is able to adapt its compute to the difficulty of the task.
This capability stems from dynamically scaling depth through fixed-point iterations and improving signal propagation. We hope these architectural modifications and the accompanying insights will support further progress on latent reasoning models.

\newpage

\section*{Acknowledgments and Disclosure of Funding}
We thank Felix Sarnthein, Albert Catalan-Tatjer, Jonas Geiping, Philipp Nazari, Carl Richardson, and Nouha Dziri for the helpful discussions and comments. Alexander Theus, Vera Milovanović, and Shlomo Libo Feigin are supported by the Max Planck ETH Center for Learning Systems (CLS). Vera Milovanović and Antonio Orvieto are supported by the AI2050 program at Schmidt Sciences. Antonio Orvieto, T. Konstantin Rusch, and Sajad Movahedi acknowledge the financial support of the
Hector Foundation.

% --- BIBLIOGRAPHY COMMANDS ---
\bibliographystyle{plainnat} % Sets the style of the references (e.g., plain, unsrt, alpha)
\bibliography{bibliography} % Points to your references.bib file (do not include the .bib extension)

\newpage
\onecolumn
\appendix

\crefalias{section}{appendix}
\crefalias{subsection}{appendix}
\crefalias{subsubsection}{appendix}

\part{Appendices}
\etocsettocstyle{}{}
\localtableofcontents
\newpage

\raggedbottom

\begin{figure}[t]
    \centering

    % --- Shared legend for all three plots ---
    \includegraphics[
    width=0.75\linewidth,
    trim=0 8pt 0 0,
    clip
]{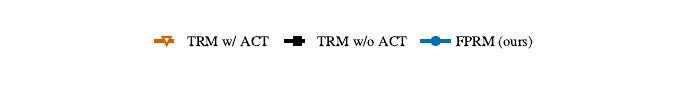}
    % -------- Left: Figure 5, with two subfigures --------
    \begin{minipage}[t]{0.66\linewidth}
        \centering

        \begin{subfigure}[t]{0.48\linewidth}
            \centering
            \includegraphics[width=\linewidth]{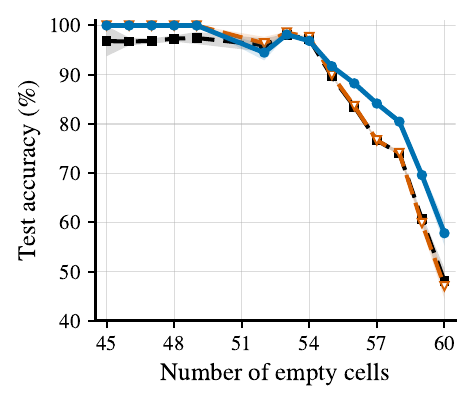}
            \caption{Accuracy vs. difficulty}
            \label{fig:accuracy_difficulty_sudoku}
        \end{subfigure}
        \hfill
        \begin{subfigure}[t]{0.48\linewidth}
            \centering
            \includegraphics[width=\linewidth]{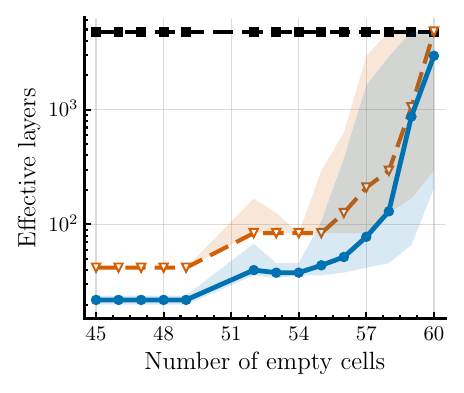}
            \caption{Inference compute vs. difficulty}
            \label{fig:compute_difficulty_sudoku}
        \end{subfigure}

        \caption{
            \looseness=-1
            \textbf{FPRM achieves (a) better accuracy, while (b) adapting more efficiently to the task difficulty.}
            %\textbf{FPRM adapts to the problem difficulty on Sudoku.}
            Difficulty is measured by the number of empty cells in the sudoku grid. The max. compute budget is matched across models (4788 effective layers). In (b): effective layers are reported as medians with 25th--75th percentile bands. The default behavior of TRM is without ACT at inference time (in black), which exhausts the max. budget for all sample difficulties.}
        \label{fig:adaptivity_sudoku}
    \end{minipage}
    \hfill
    % -------- Right: Figure 6, separate numbered figure --------
    \begin{minipage}[t]{0.32\linewidth}
        \centering

        \includegraphics[width=\linewidth]{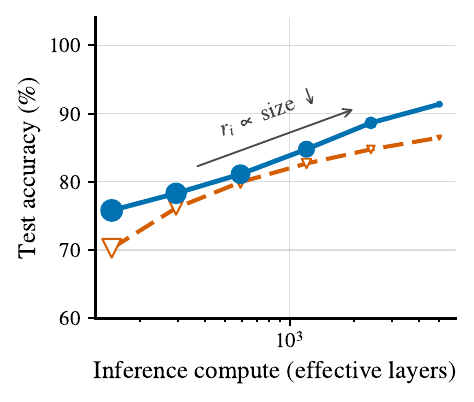}
        \vspace{0.51em}

        \caption{
            \looseness=-1
            \textbf{Test-time scaling lowers residuals while improving accuracy.} Both models run until the matched compute budget is exhausted. FPRM achieves higher accuracy across the entire range. Marker size denotes residual.}
        \label{fig:sudoku_inference_scaling}
    \end{minipage}
\end{figure}

\section{Background and Related Works}
\label{sec:background}

\paragraph{Looped models.} Looped architectures address the flexibility requirement of reasoning models by decoupling \efflayer from parameter count, providing variable computation per input without scaling the number of parameters. This makes looping a natural inductive bias for tasks that can be solved by repeatedly applying local or compositional subroutines. Early examples include Neural GPUs~\citep{kaiser2016neuralgpuslearnalgorithms} and Universal Transformers~\citep{dehghani2019universaltransformers}; more recent works show that Looped Transformers can improve length generalization, learn algorithmic structure, and approximate much deeper untied models on reasoning tasks~\citep{yang2024loopedtransformersbetterlearning,fan2025loopedtransformerslengthgeneralization,saunshi2025reasoninglatentthoughtspower,kapl2026growingloopingunifiedview,geiping2025scalingtesttimecomputelatent,giannou2023looped,kohli2026loop}. Recent recursive reasoning models such as HRM~\citep{DBLP:journals/corr/abs-2506-21734} and TRM~\citep{jolicoeurmartineau2025morerecursivereasoningtiny} instantiate this principle in compact architectures. Using small networks, these models have been able to outperform prominent LLM-based reasoning models. At the core of these methods is the idea of the need for a hierarchical looping mechanism, wherein the compute is distributed between a fast-looping component and a slow-looping component. URM \citep{DBLP:journals/corr/abs-2512-14693} follows up on these works and shows that by using short-conv, performance can be further improved. However, as we describe in this paper, previous works leave two central design axes open to improvement: First, \textit{how should a looped model decide how many iterations to run on each input at test-time?} FPRM addresses this through fixed-point halting. Second, \textit{how to make the model utilize its deeper \efflayers?} FPRM does so by switching post-norm to pre-norm and adding residual scaling. Importantly, FPRM achieves better performance than TRM and HRM without using a hierarchical structure for the loops. \looseness=-1

\paragraph{Adaptive computation.} Classical ACT methods answer the adaptivity question by learning an explicit halting rule, such as halting probabilities, per-token stopping decisions, or learned distributions over computation depth~\citep{dehghani2019universaltransformers,banino2021pondernetlearningponder}.
More recent works extend this idea by picking an adaptive depth for each token ~\citep{bae2025mixtureofrecursionslearningdynamicrecursive, song2026adaponderlmgatedponderinglanguage}. Such mechanisms can, in principle, allocate more computation to harder instances, but they all add a separate learned decision process on top of the recurrent computation itself, and this process is hard to optimize because the halting decision is discrete rather than continuous. Moreover, analyses of HRM find that ACT does not always scale the inference compute with the actual difficulty of the input \citep{ren2026reasoningmodelsreasoningguessing,ge2025hierarchicalreasoningmodelsperspectives}, which can more generally also be the case with CoT \citep{palod2025performativethinkingbrittlecorrelation}, a limitation we also observe and address by instead halting when the iteration reaches a fixed-point.

\paragraph{Deep equilibrium and fixed-point models.} An alternative is to use the convergence of the latent dynamics as the halting criterion. In this view, computation stops when consecutive iterates become sufficiently close $\|\bmz_{i+1}-\bmz_i\| \leq \epsilon$, which corresponds to convergence toward a fixed-point $\bmz^\star = f_\theta(\bmz^\star;\bmx)$. This perspective is also supported by recent mechanistic analyses of looped language models, which observe that recurrent trajectories often converge to fixed-points, suggesting that weight-tied looped models can naturally implement an implicit fixed-point computation~\citep{blayney2026mechanistic}. This view is most extensively developed in Deep Equilibrium Models~\citep{DBLP:conf/nips/BaiKK19}, which replace a finite stack of layers by the equilibrium point of a weight-tied transformation, a mechanism aligned with adaptivity~\citep{DBLP:conf/nips/GeorgievWBL24, xhonneux2024deepequilibriummodels}. Equilibrium-based halting for reasoning has therefore been explored previously. Our contribution focuses on improving signal propagation and enabling adaptive fixed-point halting in looped Transformers. However, DEQ models usually find the fixed-point via a quasi-Newton approach such as Broyden's method~\citep{Broyden1965ACO, DBLP:conf/nips/BaiKK19} or Anderson acceleration \citep{Anderson1965IterativePF, geng2023torchdeqlibrarydeepequilibrium} rather than fixed-point iterations, which makes them different from looped models. Moreover, DEQ models are difficult to optimize \citep{bai2021stabilizing, DBLP:conf/nips/AnilPLTWBKG22, jolicoeurmartineau2025morerecursivereasoningtiny}, which we address by our proposed architectural changes. Concurrent to our work, Attractor Models \citep{feinashley2026solveloopattractormodels} frame the iterations of TRM as a root-finding problem similar to DEQ, which they solve with Anderson acceleration. Additionally, they propose to optionally include a larger separate network to “guess” the initial latent. Another concurrent work is Equilibrium Reasoners (EqR) \citep{huang2026equilibriumreasonerslearningattractors}, which also adopts a fixed-point perspective and shows that it is possible to scale the compute not only depth-wise, but also by making multiple initial guesses at training and inference time, favoring attractors with a wide basin. Our contributions are largely orthogonal: we target the signal-propagation issues that limit trainable depth, replacing post-norm with pre-norm and residual scaling, and damping the iteration to suppress oscillation around the fixed-point. Moreover, in contrast to these works we do not use the hierarchical looping of TRM. 

\paragraph{Score-based methods.} Diffusion Models \citep {ho2020denoising, song2020score} and Energy-Based Models (EBMs) \citep{lecun2006tutorial} provide another way of iterative computation. An EBM learns a scalar energy function, the core idea being that the energy represents the negative log of the unnormalized density, and solutions correspond to low energy states. EBMs
generate a prediction either by descending the energy landscape \citep{belanger2016structured, belanger2017end} or by sampling with MCMC methods \citep{du2019implicit}. In comparison, diffusion models learn the score of the underlying distribution,  which can be viewed as the gradient of a time-dependent energy function. Thus the two are closely connected \citep{salimans2021should, du2023reduce}. While diffusion models typically integrate an SDE over a fixed time interval and in this sense are not adaptive, energy-based reasoning models \citep{du2022learningiterativereasoningenergy, du2024learningiterativereasoningenergy, gladstone2025energy} have an adaptive halting mechanism. In these models, the halting criterion is convergence to a local minimum of the energy landscape. Compared to these methods, we model the generative
process via fixed-point iterations of a learned operator, as opposed to descending through a learned energy function. This is a more general formulation, since finding a local minimum through gradient descent methods can be viewed as an instance of a fixed-point iteration. 

\paragraph{Signal propagation in looped models.} As unrolled looped architectures can be viewed as deep networks, they are exposed to the same signal-propagation difficulties that arise in very deep Transformers. In deep sequence models, increasing depth can make optimization harder and can prevent later layers from being effectively used, a phenomenon often discussed as the curse of depth~\citep{DBLP:conf/icml/DongCL21,DBLP:conf/nips/NociABOSL22,DBLP:journals/corr/abs-2502-05795}. 
Pre-norm is the standard remedy for these issues, but in a looped setting it removes a property the architecture relies on: \textbf{bounded activations} (\Cref{fig:ut_norm}). Bounded activations, important for stable training, are typically enforced through post-norm~\citep{geiping2025scalingtesttimecomputelatent, DBLP:journals/corr/abs-2506-21734, jolicoeurmartineau2025morerecursivereasoningtiny}. Indeed, in our own experiments a pre-norm looped model without further modification diverges in activation norm as the iteration count grows, and fails to train at the large loop counts where this growth is most severe (\Cref{fig:ut_norm}). Therefore, our FPRM combines pre-norm with residual scaling to recover bounded and stable dynamics while still allowing gradients and representations to propagate through many iterations.

\section{Proofs}
\subsection{Bound on fixed-point iterates}
\label{app:alphas-proof}

\begin{proof}
We simplify our notation by omitting the pre-norm layer in~\Cref{eq:layer_scaling}. We start by unrolling the computation across the $L$ layers within a single fixed-point iteration. By recursively applying the layer update, we obtain
\begin{align*}
\bmz_i^{2L}
=
&\ \alpha_1\cdot \bmz_i^{2L-1}
+
\beta_1\cdot f^{2L}_{\theta^{2L}}\left(\bmz_i^{2L-1}\right) \\
=
&\ \alpha_1^2\cdot \bmz_i^{2L-2}
+
\alpha_1\beta_1\cdot f^{2L-1}_{\theta^{2L-1}}\left(\bmz_i^{2L-2}\right)
+
\beta_1\cdot f^{2L}_{\theta^{2L}}\left(\bmz_i^{2L-1}\right) \\
=
&\ \alpha_1^3\cdot \bmz_i^{2L-3}
+
\alpha_1^2\beta_1\cdot f^{2L-2}_{\theta^{2L-2}}\left(\bmz_i^{2L-3}\right)
+
\alpha_1\beta_1\cdot f^{2L-1}_{\theta^{2L-1}}\left(\bmz_i^{2L-2}\right)
+
\beta_1\cdot f^{2L}_{\theta^{2L}}\left(\bmz_i^{2L-1}\right) \\
\vdots \\
=
&\ \alpha_1^{2L}\cdot \bmz_i^0
+
\beta_1\cdot
\left(
\sum_{j=0}^{2L-1}
\alpha_1^j\cdot
f^{2L-j}_{\theta^{2L-j}}
\left(\bmz_i^{2L-j-1}\right)
\right).
\end{align*}
Now, substituting this expression into the fixed-point update gives
\begin{align*}
\bmz_{i+1}^0
=
&\ \alpha_2\cdot \bmz_i^{2L}
+
\beta_2\cdot \bmx \\
=
&\ \alpha_2\cdot
\left(
\alpha_1^{2L}\cdot \bmz_i^0
+
\beta_1\cdot
\sum_{j=0}^{2L-1}
\alpha_1^j\cdot
f^{2L-j}_{\theta^{2L-j}}
\left(\bmz_i^{2L-j-1}\right)
\right)
+
\beta_2\cdot \bmx \\
=
&\ \alpha_2\alpha_1^{2L}\cdot \bmz_i^0
+
\alpha_2\beta_1\cdot
\left(
\sum_{j=0}^{2L-1}
\alpha_1^j\cdot
f^{2L-j}_{\theta^{2L-j}}
\left(\bmz_i^{2L-j-1}\right)
\right)
+
\beta_2\cdot \bmx.
\end{align*}
For compactness, define
$
\rho=\alpha_2\alpha_1^{2L}
$
and
$
\bms_i=
\sum_{j=0}^{2L-1}
\alpha_1^j\cdot
f^{2L-j}_{\theta^{2L-j}}
\left(\bmz_i^{2L-j-1}\right).
$
Then the fixed-point iteration can be written as
$
\bmz_{i+1}^0
=
\rho\cdot \bmz_i^0
+
\alpha_2\beta_1\cdot \bms_i
+
\beta_2\cdot \bmx.
$
Unrolling this recursion over fixed-point iterations gives
\begin{align*}
\bmz_{i+1}^0
=
&\ \rho^{i+1}\cdot \bmz_0^0
+
\beta_2
\left(
\sum_{k=0}^{i}\rho^k
\right)
\cdot \bmx
+
\alpha_2\beta_1
\left(
\sum_{k=0}^{i}
\rho^k\cdot \bms_{i-k}
\right).
\end{align*}
Since $0\leq \alpha_1,\alpha_2<1$, we have $0\leq \rho<1$. Therefore, the geometric series is convergent. Taking norms and using the boundedness of each layer map, we get
\begin{align*}
\left\Vert \bmz_{i+1}^0\right\Vert
\leq
&\ \rho^{i+1}\left\Vert \bmz_0^0\right\Vert
+
\beta_2
\left(
\sum_{k=0}^{i}\rho^k
\right)
\left\Vert \bmx\right\Vert
+
\alpha_2\beta_1
\left(
\sum_{k=0}^{i}\rho^k
\right)
\left(
\sum_{j=0}^{2L-1}\alpha_1^j
\right)c_f.
\end{align*}
Letting $i\to\infty$, the first term vanishes and the two geometric sums converge, which gives
\begin{align*}
\limsup_{i\to\infty}
\left\Vert \bmz_i^0\right\Vert
\leq
&\ \frac{\beta_2}{1-\rho}
\left\Vert \bmx\right\Vert
+
\frac{\alpha_2\beta_1}{1-\rho}
\left(
\frac{1-\alpha_1^{2L}}{1-\alpha_1}
\right)c_f.
\end{align*}
Substituting back $\rho=\alpha_2\alpha_1^{2L}$, we obtain
\begin{align*}
\limsup_{i\to\infty}
\left\Vert \bmz_i^0\right\Vert
\leq
&\ \frac{\beta_2}{1-\alpha_2\alpha_1^{2L}}
\left\Vert \bmx\right\Vert
+
\frac{\alpha_2\beta_1\left(1-\alpha_1^{2L}\right)}
{\left(1-\alpha_2\alpha_1^{2L}\right)\left(1-\alpha_1\right)}
c_f.
\end{align*}
We now set
$
\beta_2=1-\alpha_2\alpha_1^{2L}.
$
This makes the coefficient of $\left\Vert \bmx\right\Vert$ equal to $1$. Furthermore, setting
$
\beta_1=
\frac{\beta_2\left(1-\alpha_1\right)}
{\left(1-\alpha_1^{2L}\right)}
$
makes the coefficient of $c_f$ equal to $\alpha_2$. Therefore,
\begin{align*}
\limsup_{i\to\infty}
\left\Vert \bmz_i^0\right\Vert
\leq
\left\Vert \bmx\right\Vert
+
\alpha_2\cdot c_f.
\end{align*}
In particular, if the fixed-point iteration converges to $\bmz_\infty^0$, then
\begin{align*}
\left\Vert \bmz_\infty^0\right\Vert
\leq
\left\Vert \bmx\right\Vert
+
\alpha_2\cdot c_f.
\end{align*}
This completes the proof.
\end{proof}

\subsection{Contractive mapping}
\label{app:damped-fp-convergence-proof}
\begin{proof}
We first show that $f_{\theta}(.;\bmx)$ is a contraction with respect to $\bmz$. For any $\bmz,\bmz'$, using the Lipschitzness of the $L$-layer model, we have
\begin{align*}
\left\Vert
f_{\theta}(\bmz;\bmx)
-
f_{\theta}(\bmz';\bmx)
\right\Vert
\leq
&\
\lambda_f
\left\Vert
\left(\alpha_2\cdot \bmz+\beta_2\cdot \bmx\right)
-
\left(\alpha_2\cdot \bmz'+\beta_2\cdot \bmx\right)
\right\Vert \\
=
&\
\alpha_2\lambda_f
\left\Vert
\bmz-\bmz'
\right\Vert.
\end{align*}
Therefore, if $0\leq \alpha_2\lambda_f<1$, the map $f_{\theta}(.;\bmx)$ is strictly contractive. By the Banach fixed-point theorem, it has a unique fixed-point $\bmz^\star$, and the iteration
$
\bmz_{i+1}=f_{\theta}(\bmz_i;\bmx)
$
converges to $\bmz^\star$.

We now prove the residual bound. Since $\bmz_{i+1}=f_{\theta}(\bmz_i;\bmx)$, the residual at iteration $i$ can be written as
\begin{align*}
\left\Vert
f_{\theta}(\bmz_i;\bmx)-\bmz_i
\right\Vert
=
\left\Vert
\bmz_{i+1}-\bmz_i
\right\Vert.
\end{align*}
Using the contraction property of $f_{\theta}(.;\bmx)$, we get
\begin{align*}
\left\Vert
\bmz_{i+1}-\bmz_i
\right\Vert
=
&\
\left\Vert
f_{\theta}(\bmz_i;\bmx)
-
f_{\theta}(\bmz_{i-1};\bmx)
\right\Vert \\
\leq
&\
\alpha_2\lambda_f
\left\Vert
\bmz_i-\bmz_{i-1}
\right\Vert.
\end{align*}
Applying this inequality recursively gives
\begin{align*}
\left\Vert
\bmz_{i+1}-\bmz_i
\right\Vert
\leq
\left(\alpha_2\lambda_f\right)^i
\left\Vert
\bmz_1-\bmz_0
\right\Vert.
\end{align*}
Since $\bmz_1=f_{\theta}(\bmz_0;\bmx)$, we obtain
\begin{align*}
\left\Vert
f_{\theta}(\bmz_i;\bmx)-\bmz_i
\right\Vert
\leq
\left(\alpha_2\lambda_f\right)^i
\left\Vert
f_{\theta}(\bmz_0;\bmx)-\bmz_0
\right\Vert.
\end{align*}
This completes the proof.
\end{proof}

% \label{app:trunc-bptt-proof}

\subsection{Mitigating oscillation through damping}
\label{app:damping-proof}
\begin{proof}
We first show that the fixed-points of $g_{\eta,\theta}(\,\cdot\,;\bmx)$ and $f_\theta(\,\cdot\,;\bmx)$ coincide. Since
\begin{align*}
    g_{\eta,\theta}(\bmz;\bmx) - \bmz \;=\; \eta\bigl(f_\theta(\bmz;\bmx) - \bmz\bigr),
\end{align*}
and $\eta > 0$, we have $g_{\eta,\theta}(\bmz;\bmx) = \bmz$ if and only if $f_\theta(\bmz;\bmx) = \bmz$.

We now study the local stability of the damped iteration around $\bmz^\star$. The Jacobian of $g_{\eta,\theta}(\,\cdot\,;\bmx)$ at $\bmz^\star$ is
\begin{align*}
    \frac{\partial g_{\eta,\theta}}{\partial \bmz}(\bmz^\star;\bmx) \;=\; (1-\eta)\matI + \eta\matJ,
\end{align*}
so for every eigenvalue $\lambda_i$ of $\matJ$, the corresponding eigenvalue of the damped Jacobian is
\begin{align*}
    \mu_i(\eta) \;=\; 1 - \eta + \eta\lambda_i \;=\; 1 + \eta(\lambda_i - 1).
\end{align*}
Local asymptotic stability is implied by $|\mu_i(\eta)| < 1$ for all $i$. Writing $\lambda_i = a_i + \mathrm{i}\,b_i$ with $a_i = \Re(\lambda_i)$,
\begin{align*}
    |\mu_i(\eta)|^2 \;=\; 1 + 2\eta(a_i - 1) + \eta^2 |\lambda_i - 1|^2.
\end{align*}
Hence $|\mu_i(\eta)| < 1$ if and only if $\eta |\lambda_i - 1|^2 < 2(1 - a_i)$, i.e.,
\begin{align*}
    0 \;<\; \eta \;<\; \frac{2(1 - \Re(\lambda_i))}{|\lambda_i - 1|^2}.
\end{align*}
By assumption $\Re(\lambda_i) < 1$ for every $i$, so each upper bound is strictly positive. Setting
\begin{align*}
    \eta_0 \;=\; \min\!\left\{\,1,\; \min_i \frac{2(1 - \Re(\lambda_i))}{|\lambda_i - 1|^2}\,\right\} \;>\; 0,
\end{align*}
we obtain $|\mu_i(\eta)| < 1$ for every $i$ and every $\eta \in (0, \eta_0)$. Therefore $\bmz^\star$ is locally asymptotically stable under the damped iteration $\bmz_{t+1} = g_{\eta,\theta}(\bmz_t;\bmx)$, and the iterates converge to $\bmz^\star$ from any sufficiently close initialization.
\end{proof}

\subsection{Error of truncated-BPTT}
\label{app:trunc-bptt-proof}

\begin{proof}
Since $\|\mathbf{J}\|_2=\sigma<1$, the Neumann series is convergent, and we have
$
\left(\mathbf{I}-\mathbf{J}\right)^{-1}
=
\sum_{j=0}^{\infty}\mathbf{J}^j.
$
Therefore, the error of the $k$-term truncated approximation is
\begin{align*}
\left\Vert
\left(\mathbf{I}-\mathbf{J}\right)^{-1}
-
\sum_{j=0}^{k-1}\mathbf{J}^j
\right\Vert_F
&=
\left\Vert
\sum_{j=k}^{\infty}\mathbf{J}^j
\right\Vert_F \\
&\leq
\sum_{j=k}^{\infty}
\left\Vert
\mathbf{J}^j
\right\Vert_F.
\end{align*}
Using the relation $\|\mathbf{A}\|_F\leq \sqrt{D}\|\mathbf{A}\|_2$ for $\mathbf{A}\in\mathbb{R}^{D\times D}$, together with submultiplicativity of the spectral norm, we get
\begin{align*}
\left\Vert
\mathbf{J}^j
\right\Vert_F
&\leq
\sqrt{D}
\left\Vert
\mathbf{J}^j
\right\Vert_2 \\
&\leq
\sqrt{D}
\left\Vert
\mathbf{J}
\right\Vert_2^j \\
&=
\sqrt{D}\cdot \sigma^j.
\end{align*}
Substituting this into the previous inequality gives
\begin{align*}
\left\Vert
\left(\mathbf{I}-\mathbf{J}\right)^{-1}
-
\sum_{j=0}^{k-1}\mathbf{J}^j
\right\Vert_F
&\leq
\sqrt{D}
\sum_{j=k}^{\infty}
\sigma^j \\
&=
\sqrt{D}\cdot
\frac{\sigma^k}{1-\sigma}.
\end{align*}
Thus, the approximation error decays as $\mathcal{O}(\sigma^k)$.

To make the corresponding gradient statement explicit, let
$
\mathbf{P}
=
\frac{\partial f_\theta}{\partial \theta}(\mathbf{z}^\star;\mathbf{x})
$
and
$
\boldsymbol{\delta}
=
\frac{\partial \mathcal{L}}{\partial \mathbf{z}^\star}.
$
The exact implicit gradient is
$
\nabla_\theta \mathcal{L}
=
\mathbf{P}^\top
\left(\mathbf{I}-\mathbf{J}\right)^{-\top}
\boldsymbol{\delta},
$
whereas the $k$-step truncated BPTT gradient is
$
\widehat{\nabla}_\theta^{(k)}\mathcal{L}
=
\mathbf{P}^\top
\left(
\sum_{j=0}^{k-1}
\left(\mathbf{J}^\top\right)^j
\right)
\boldsymbol{\delta}.
$
Therefore,
\begin{align*}
\left\Vert
\nabla_\theta \mathcal{L}
-
\widehat{\nabla}_\theta^{(k)}\mathcal{L}
\right\Vert_2
&\leq
\left\Vert \mathbf{P}\right\Vert_2
\left\Vert
\left(\mathbf{I}-\mathbf{J}\right)^{-\top}
-
\sum_{j=0}^{k-1}
\left(\mathbf{J}^\top\right)^j
\right\Vert_2
\left\Vert \boldsymbol{\delta}\right\Vert_2 \\
&\leq
\left\Vert \mathbf{P}\right\Vert_2
\left(
\sum_{j=k}^{\infty}
\left\Vert \mathbf{J}\right\Vert_2^j
\right)
\left\Vert \boldsymbol{\delta}\right\Vert_2 \\
&=
\left\Vert \mathbf{P}\right\Vert_2
\frac{\sigma^k}{1-\sigma}
\left\Vert \boldsymbol{\delta}\right\Vert_2.
\end{align*}
Hence, the truncated BPTT gradient error also decays exponentially with the number of backward passes $k$. This completes the proof.
\end{proof}

\section{A Toy Failure Mode for Recurrent Post-norm}
\label{sec:pre_post_toy}

In this section, we provide a toy example to showcase a failure mode of post-norm in a small setting. We take random $\bmx,\bmy\in\mathbb{R}^{n\times d}$ to be the input-output pair of sequences, with sequence length $n=100$ and hidden-size $d=2$, and set $\bmz^0_0=\bmx$. Let $f_\theta(\bmz;\bmx)$ be a neural network with a single sub-layer $f_{\theta^1}^1(\bmz)$ defined as:
\begin{equation*}
    f_{\theta^1}^1(\bmz)= \bmz W(\bmw),
\end{equation*}
where we define the rank-one map:
\begin{equation*}
    W(\bmw)=\bmw\mathbf{1}^{\top}\in \mathbb{R}^{2\times 2},
    \qquad
    \bmw=(w_1,w_2)^\top,
\end{equation*}
with $\theta^1=\bmw$.

The \textbf{post-norm recurrence} is defined as
\begin{equation*}
    \bmz_{i+1}^0 = \mathrm{Norm}_{\mathrm{post}}\left(\bmz_{i}^0 + f_{\theta^1}\left(\bmz_{i}^0\right) \right),
\end{equation*}
while the \textbf{pre-norm recurrence} is defined as

\begin{equation*}
    \bmz_{i+1}^0 = \left(1 - \beta\right) \cdot \bmz_{i}^0 + \beta \cdot f_{\theta^1}\left(\mathrm{Norm}_{\mathrm{post}}\left(\bmz_{i}^0\right)\right),
    \qquad
    \beta=\frac12 .
\end{equation*}

\Cref{fig:pre_post_toy} gives a minimal version of the normalization trade-off discussed in the main text based on this toy model. For both models we sweep a $200\times 200$ grid over $\bmw\in \mathcal{G}= [-5,5]^2$ and plot

\begin{equation*}
    \Delta\mathcal{L}(\bmw)
    =
    \mathcal{L}(\bmw)-\min_{\bmv\in\mathcal{G}}\mathcal{L}(\bmv),
\end{equation*}
where we define the loss as
\begin{equation*}
    \mathcal{L}(\bmw)
    =
    \frac{1}{nd}
    \left\|
        \bmz_{20}^0(\bmw)-\bmy
    \right\|_F^2,
\end{equation*}
i.e., at the $20^{th}$ \efflayer.

The figure illustrates that \textbf{boundedness does not imply trainability}.  Post-norm keeps the recurrent state bounded by construction: after every step, each row is projected back onto the unit sphere.  However, the same projection also removes radial information at every iteration.  In this two-parameter slice, the resulting loss is organized into thin angular sectors with sharp ridges and narrow low-loss regions.  Thus a random initialization of $\bmw$ is likely to start in a bounded but poorly conditioned part of the landscape, where the gradient does not point into a useful basin. This is the toy analogue of the optimization difficulty of recurrent post-norm layers.

The right panel is not bare pre-norm; it is pre-norm with residual scaling.  This matters because naive pre-norm removes the projection that controls the recurrent state and can lead to activation growth (\Cref{fig:ut_activation_norm}).  With the scaled update, each row satisfies
\begin{equation*}
\|\bmz_{i + 1}^{0}\|_2
    \le
    (1-\beta)\|\bmz_{i}^{0}\|_2+\beta\|W(w)\|_2.
\end{equation*}
So the toy dynamics remain bounded while preserving a live residual stream.  The broader low-loss region in \Cref{fig:pre_post_toy} is therefore consistent with the architecture we use in \Cref{thm:bounded-fp-iterates}: pre-normalization improves signal propagation, while residual scaling replaces the boundedness mechanism that post-normalization provided.
\begin{figure}
    \centering
    \includegraphics[width=0.9\linewidth]{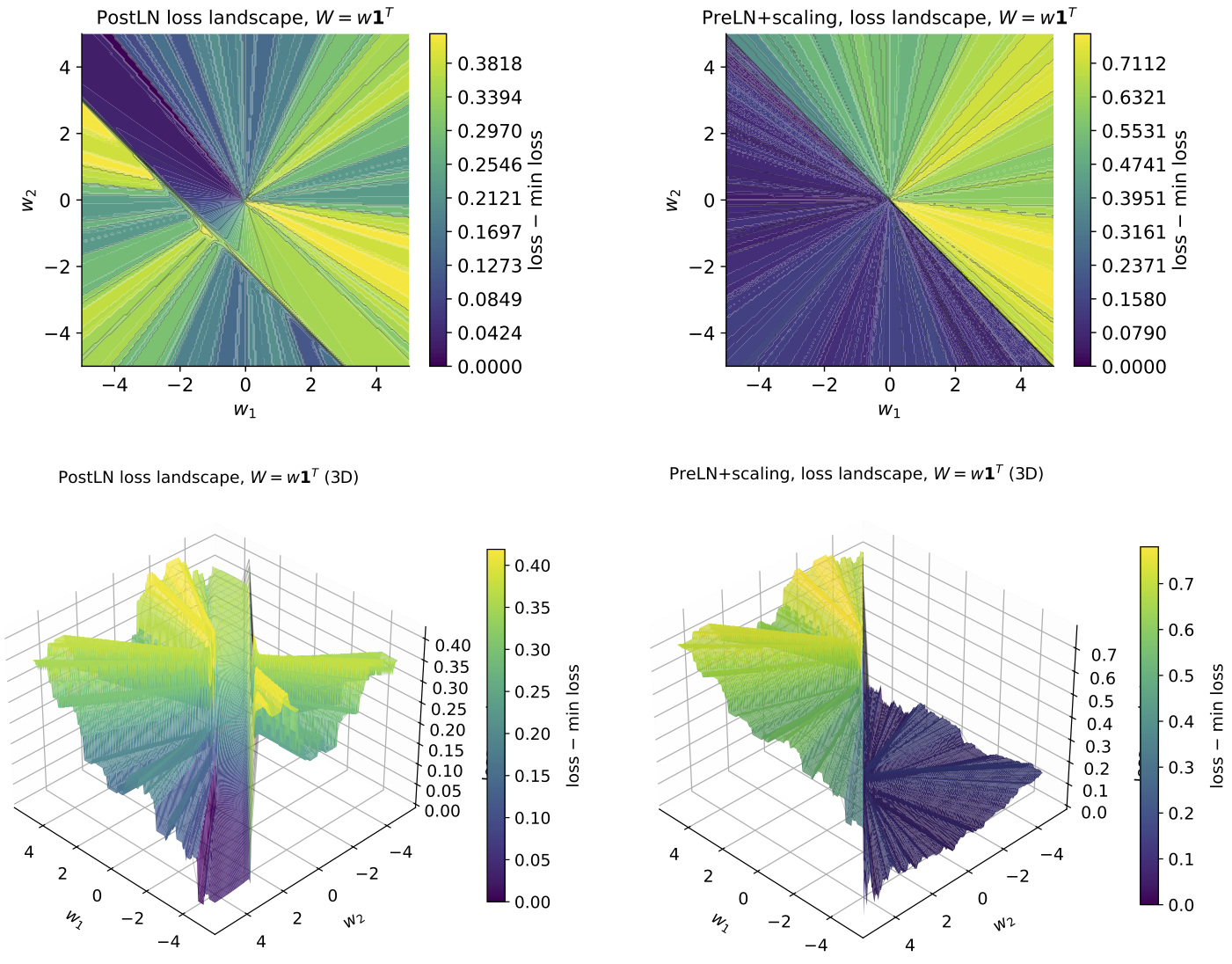}
    \caption{Landscape visualization for the setup proposed in \Cref{sec:pre_post_toy}.}
    \label{fig:pre_post_toy}
\end{figure}

\section{Training Fixed-Point Models}
\label{sec:fp-optimizer}
\subsection{Truncated Back-propagation Through Time}
One advantage of contractive fixed-point models is that they can be trained using truncated back-propagation through time (BPTT)~\citep{DBLP:conf/nips/GengZBWL21}. Let
$\matJ = \frac{\partial f_\theta}{\partial \bmz}(\bmz^\star;\bmx)$
and
$\matP = \frac{\partial f_\theta}{\partial \theta}(\bmz^\star;\bmx)$
denote the Jacobians of $f_\theta$ at the fixed-point with respect to the state and the parameters, respectively. Following the implicit function theorem we can write the gradient w.r.t. the parameters of the model as~\citep{DBLP:conf/nips/BaiKK19}:

\begin{equation}
    \label{eqn:ift}
    \frac{d\bmz^\star}{d\theta} \;=\; (\matI - \matJ)^{-1}\,\matP.
\end{equation}

The Neumann series $(\matI-\matJ)^{-1} = \sum_{j\geq 0}\matJ^{\,j}$ converges, assuming $f_\theta(\bmz;\bmx)$ is contractive. Therefore, truncating the series at depth $k$ yields the estimate:

\begin{equation}
    \label{eqn:neumann_trunc}
    \frac{d\bmz^\star}{d\theta} \;\approx\; {\textstyle\sum_{j=0}^{k-1}} \matJ^{\,j}\, \matP,
\end{equation}
which is closely related to Jacobian-free backpropagation, where the full implicit linear solve is replaced by cheaper approximate gradients~\citep{fung2022jfb}. The depth of truncation $k$ trades off computation against accuracy. The following proposition bounds the resulting error under contractivity.

\begin{proposition}[Exponential decay of truncated-BPTT error]
\label{prop:trunc-bptt-error}
Let $\bmz^\star = f_\theta(\bmz^\star; \bmx)$ and $\matJ = \frac{\partial f_\theta}{\partial \bmz}(\bmz^\star;\bmx) \in \mathbb{R}^{D \times D}$. If $\matJ$ is contractive in spectral norm, $\|\matJ\|_2 = \sigma < 1$, then for every $k \geq 0$,
\begin{equation*}
    \left\Vert (\matI - \matJ)^{-1} - {\textstyle\sum_{j=0}^{k-1}} \matJ^{\,j} \right\Vert_F \;\leq\; \sqrt{D}\,\frac{\sigma^k}{1 - \sigma}.
\end{equation*}
\end{proposition}
\noindent The proof is in~\Cref{app:trunc-bptt-proof}. An essentially equivalent result appears in the proof of Theorem 2 of~\citet{DBLP:conf/nips/GengZBWL21}.

\Cref{prop:trunc-bptt-error} allows for a fixed memory footprint during training, essentially decoupling the number of loops from the memory complexity of the model. In the same spirit, HRM~\citep{DBLP:journals/corr/abs-2506-21734} and TRM~\citep{jolicoeurmartineau2025morerecursivereasoningtiny} approximate the gradient with a small number of backward passes at the fixed-point, although TRM argues that the fixed-point condition is unnecessary in practice. 

\subsection{Deep Supervision} We adopt a similar deep supervision mechanism as HRM \citep{DBLP:journals/corr/abs-2506-21734} and TRM \citep{jolicoeurmartineau2025morerecursivereasoningtiny}. Let $T_{\text{sup}}$ denote the deep supervision interval. After every $T_{\text{sup}}$ iterations, the intermediate activations of the model are decoded through the output head, the loss is computed, and truncated-BPTT is performed through the $k$ latest iterations. Then, the computation graph is detached from the previous step. For each sequence, this process is continued until the fixed-point of the input is reached. The number of backward passes per forward pass is therefore $k / T_{\text{sup}}$. In our model, we set $T_{\text{sup}}=k$, while in TRM and HRM, $T_{\text{sup}}$ is usually set to a larger number. However, we observe that in practice, FPRM performs a smaller number of forward and backward passes during training compared to TRM, lowering the training cost.

\subsection{Fixed-point Solvers}  Let $\bmz_i \in \mathbb{R}^{B \times T \times d}$ denote the $i^{th}$ latent representation (with $B$ denoting the batch index, $T$ the sequence index, and $d$ the hidden size), and $\bmz_{i+1}$ the next latent representation. We index the $b^{th}$ batch dimension as $\bmz_i[b]$. Convergence is measured per sample $\bmz_i[b]$ by the relative $L_\infty$ norm of the residual,
\begin{equation*}
    \bmr_i[b] = \frac{\left\Vert \bmz_{i + 1}[b] - \bmz_i[b] \right\Vert_\infty}{\left\Vert \bmz_{i+1}[b] \right\Vert_\infty + \epsilon} \in \mathbb{R}.
\end{equation*}
 A sample is declared converged when $r_i[b] < \tau$. In practice, we set $\tau$ to $0.1$. But we observe that for a reasonably small choice of $\tau$, the model is not sensitive to this value. Two safeguards bound the loop: \textbf{(1)} a hard cap on the number of iterations, and \textbf{(2)} early termination if the adaptive step size collapses below a minimum.

\textbf{Other fixed-point solvers.} In preliminary experiments, we also tried Anderson acceleration and Broyden’s method~\citep{geng2023torchdeqlibrarydeepequilibrium}. In our setup, both complicated training and made convergence more difficult, while introducing additional computational and bookkeeping overhead. These observations motivated our use of damped fixed-point iteration, although a systematic comparison of solver choices remains future work.

\section{Further Details About the Architectures}\label{architecture_details}
\paragraph{Depth-wise convolutions.} Depth-wise convolutions have proven effective in improving the performance of looped models, with low computational and parameter overhead~\citep{DBLP:journals/corr/abs-2602-02156, DBLP:journals/corr/abs-2512-14693}. Therefore, in FPRM we apply depth-wise convolutions on the latent representations at the beginning of each loop, which we find to be most effective. An overview of FPRM is available in~\Cref{fig:lofat-architecture}. We consider both 1D and 2D convolutions, and we find the 2D variant to be more effective at 2-dimensional tasks such as Sudoku-Extreme and ARC, while the 1D variant is essential in state-tracking, for both FPRM and TRM (\Cref{fig:state-tracking-per-k}).

\section{Dataset Description}
\label{sec:data-description}
In the following, we provide a brief description for each dataset used in this paper. We refer the reader to the cited literature for more information.

\paragraph{Sudoku-Extreme.} The task consists of exceptionally challenging, partially filled $9\times9$ sudoku puzzles with unique solutions, introduced by~\citet{DBLP:journals/corr/abs-2506-21734}. Each sample is flattened into a sequence of length $81$. The train data consists of $1000$ unique samples, each augmented $1000$ times, giving a total of \textasciitilde$1$M train samples. The test data contains $422{,}786$ samples. The evaluation metric is exact (sequence) accuracy.

\paragraph{Maze-Hard.} The task consists of difficult $30\times 30$ shortest-path maze puzzles with unique solutions, introduced by~\citet{DBLP:journals/corr/abs-2506-21734}. The train data and the test data each consist of $1000$ unique samples. Following~\citet{DBLP:journals/corr/abs-2506-21734}, we do not use augmentation for this problem. The evaluation metric is exact (sequence) accuracy.

\paragraph{ARC-1 and ARC-2.} The Abstraction and Reasoning Corpus-1 (ARC-1) and -2 (ARC-2), introduced by~\citet{chollet2025arcprize2024technical}, aim to assess the ability of the model to solve novel problems from minimal examples. ARC-1 consists of 2-3 2D grid-based input-output demonstration pairs with variable size (up to $30{\times}30$), through which the model is supposed to learn an underlying transformation rule and apply it to a held-out sample. The train data and the test data each contain \textasciitilde$400$ samples. ARC-2 has characteristics similar to those of ARC-1, but with much more challenging problems involving several complex transformations, which are less susceptible to brute-force solutions. The benchmark contains \textasciitilde$1000$ training samples and \textasciitilde$360$ test samples. The evaluation metric is exact (sequence) pass@2 (top-2 predictions) accuracy. Importantly, pretrained reasoning LLMs with CoT often struggle with these tasks. For example, DeepSeek-R1 ($671$B model) achieves $15.8\%$ on ARC-1 and $1.3\%$ on ARC-2. Claude 3.7 Sonnet 16K achieves $28.6\%$ on ARC-1 and $0.7\%$ on ARC-2, as reported by~\cite{jolicoeurmartineau2025morerecursivereasoningtiny}.

\paragraph{State tracking.} The $A_5$ and $S_5$ state-tracking tasks, introduced by~\citet{merrill2025illusionstatestatespacemodels}, are algorithmic benchmarks based on permutation composition. Each sample consists of an initial state and a sequence of $k$ update permutations; the model must apply the updates in order and predict the permutation obtained by composing all input permutations in order. $A_5$ uses the alternating group on five elements, i.e., the subgroup of even permutations, while $S_5$ uses the full symmetric group on five elements. These tasks are useful proxies for stateful reasoning problems such as entity tracking, code execution, and game-state tracking, since solving them requires learning a composable update rule rather than memorizing computations at fixed lengths. We train on sequences containing up to 32 updates and evaluate out-of-distribution length generalization on sequences containing up to 128 updates. The evaluation metric is the exact final-state accuracy.

\section{Description of~\Cref{fig:fprm_trm_halting_points}}
\label{sec:fig1-descript}

In this figure, we categorize the puzzles into three groups: easy, medium, and hard. The grouping is based on difficulty, which is measured by the number of empty cells in the puzzle. The sample sizes for each difficulty level are balanced and set at around~1000 samples. For FPRM, we mark the halting decision for the entire group based on the residual of the group: if the mean residual is smaller than a pre-determined threshold (set to~0.1), then the model makes the halting decision. For TRM, the halting decision is marked when the ACT module signals halting for more than half of the samples. For the sake of exposition, we exclude the hardest puzzles from the groups, since they fail to halt at the current maximum set budget of~10000 effective layers.

\section{Additional Experiments}

\begin{table}[t]
  \caption{\textbf{Results across seeds.} Test accuracy (\%) over independently trained seeds. We report the mean $\pm$ standard deviation along with the best and worst seed. Sudoku-Extreme and Maze-Hard report sequence accuracy. State tracking reports final-state accuracy at sequence length 128 (trained at length 32). The best mean accuracy per task is in bold.}
  \label{tab:seeds}
  \centering
  \begin{tabular}{lccc}
    \toprule
    Model & Mean $\pm$ std & Best & Worst \\
    \midrule
    \multicolumn{4}{l}{\textit{Sudoku-Extreme}} \\
    TRM       & $66.9 \pm 5.7$            & 72.6 & 59.0 \\
    FPRM      & {\boldmath$80.6 \pm 4.4$} & 85.0 & 75.1 \\
    \midrule
    \multicolumn{4}{l}{\textit{Maze-Hard}} \\
    TRM       & $84.9 \pm 2.1$            & 87.3 & 83.3 \\
    FPRM      & {\boldmath$85.9 \pm 0.9$} & 86.8 & 85.0 \\
    \midrule
    \multicolumn{4}{l}{\textit{State tracking, $A_5$}} \\
    TRM       & $45.8 \pm 3.9$            & 50.2 & 42.9 \\
    TRM+conv  & $91.4 \pm 2.3$            & 93.4 & 88.9 \\
    FPRM      & {\boldmath$98.1 \pm 2.2$} & 99.9 & 95.5 \\
    \midrule
    \multicolumn{4}{l}{\textit{State tracking, $S_5$}} \\
    TRM       & $39.4 \pm 1.9$            & 41.0 & 37.3 \\
    TRM+conv  & $97.2 \pm 2.5$            & 99.8 & 94.9 \\
    FPRM      & {\boldmath$98.8 \pm 0.9$} & 99.7 & 97.9 \\
    \bottomrule
  \end{tabular}
\end{table}

\subsection{Variance across training seeds}
\label{app:variance}

In order to investigate the effect of randomness on the performance of FPRM and TRM, we performed the experiments on the Sudoku, Maze, and state-tracking datasets using three separate seeds. Table~\ref{tab:seeds} reports results. Note that the FPRM evaluations were done with $\gamma=0.99$, which makes the fixed-point optimizer converge faster and lowers the inference-time compute, at the cost of lower performance compared to the reported numbers in~\Cref{tab:sudoku-maze}. The median number of iterations for FPRM in Sudoku is around 400-700 across seeds, depending on the sample difficulty, while in TRM the number of iterations is fixed to 336.

The state tracking evaluations are reported at the longest evaluated sequence length (128, trained at length 32). At this length, FPRM uses an average of 120$\pm$18 iterations per sample on $A_5$ and 140$\pm$14 on $S_5$, compared with a fixed 320 iterations for TRM. As we observe, FPRM is the best performing model in terms of average performance across tasks, with the lowest performance variance in all 4 experiments.

\subsection{Fixed-point residuals and halting}
\label{app:res-and-halting}
\begin{figure}[h]
    \centering
    \includegraphics[width=0.5\linewidth]{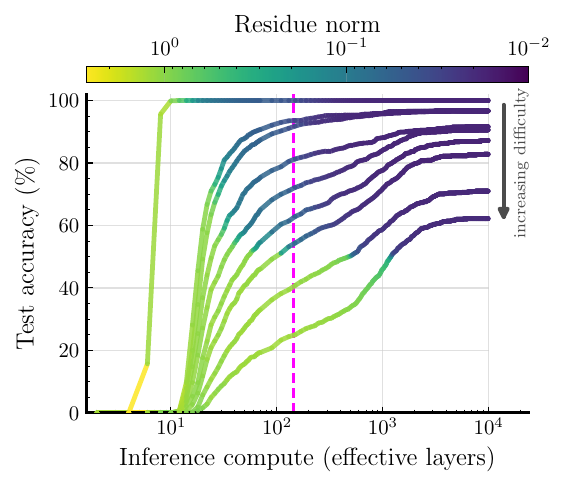}
    \caption{\textbf{FPRM allocates more compute to harder problems.} Harder inputs need more iterations before halting and peak beyond the training compute limit (dashed line); color shows residual norm.}
    \label{fig:fprm_scaling_convergence}
\end{figure}
We provide the test accuracy and the fixed-point residuals achieved by FPRM as a function of \efflayer in~\Cref{fig:fprm_scaling_convergence}. The residuals for more difficult problems decay at a much slower rate, indicating that they demand more compute. Furthermore, accuracy stops improving at roughly the same \efflayer where the residual plateaus, supporting the use of fixed-points as a halting criterion. 

\begin{figure}[t]
        \centering
        \includegraphics[width=0.5\linewidth]{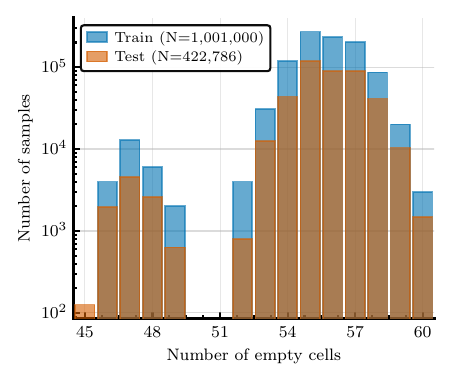}
        \caption{\textbf{Sudoku-Extreme dataset is imbalanced.} The number of samples per difficulty level (number of empty cells).}
        \label{fig:fprm_empty_cells}
\end{figure}

\subsection{Ablation on TRM}
\label{app:trm-ablation}
\subsubsection{How to effectively spend loops in TRM?} \label{sec:trm_scaling_LH_vs_deep_sup}

\begin{figure}[h]
    \centering
    \includegraphics[width=0.5\linewidth]{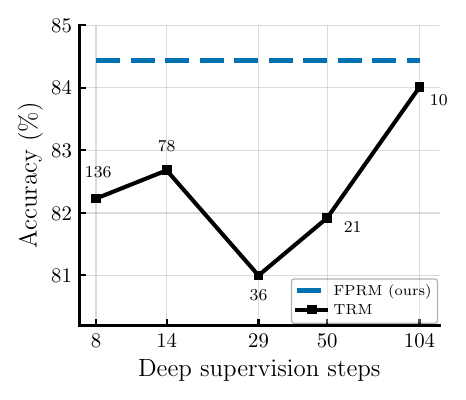}
    \caption{\textbf{The optimal way to spend the fixed looping compute is to maximize deep supervision steps.} The numbers next to markers are inner recurrence depths per deep supervision step. The total number of \efflayers is approximately the same across all configurations of TRM and FPRM on the Sudoku-Extreme task.}
    \label{fig:trm_deep_superv_vs_LH}
\end{figure}

The default TRM uses $16$ deep-supervision steps (outer loops) and variable L- and H-steps (inner loops). The L-steps outnumber the H-steps, typically by about $2\times$. However, other configurations for the number of loops allocated to deep supervision vs. inner loops are possible. We test the performance of other configurations with experiments shown in~\Cref{fig:trm_deep_superv_vs_LH}. We fix the L-to-H ratio at $2$ and vary deep-supervision steps (segments) against per-segment recurrence depth (inner loops, shown with numbers next to the black markers in~\Cref{fig:trm_deep_superv_vs_LH}). We measure test accuracy as a function of the number of deep-supervision steps, with a fixed inference budget of approximately $1040$ steps. This budget also matches the max. number of effective layers reached by the baseline FPRM on this task. This isolates how a fixed inference budget is best allocated: toward more outer refinement steps or deeper inner recurrence. \Cref{fig:trm_deep_superv_vs_LH} shows that the budget is best spent on outer, deep-supervision steps. This matches the finding that, in TRM's post-norm Transformer, the gains from added \efflayer are smaller than those of FPRM (\Cref{fig:sudoku_inference_scaling}). Fewer \efflayers per segment is therefore the better strategy at a fixed compute budget. We adopt it for all experiments where we scale TRM compute (\efflayers).

\subsubsection{Signal propagation in TRM}
\label{app:trm-sig-prop}
Here, we investigate whether the signal propagation remedies proposed in
FPRM (pre-norm and residual scaling) also help in the cases of a looped
Transformer model (1) with hierarchy (TRM) and (2) without hierarchy.
We examine both task performance under increased deep supervision and
signal propagation at initialization.

\paragraph{Task performance under increased deep supervision steps.}
For both models, to exacerbate signal propagation issues, we increase the number of deep supervision steps, effectively making the unrolled looped model deeper. Based on the experiment results on the Sudoku-Extreme task, presented in~\Cref{tab:trm_ablation} we conclude that in the case of post-norm, the average performance remains the same when removing the hierarchy with the same number of deep supervision steps, but drops significantly when increasing the number of deep supervisions steps. This can be attributed to an increase in the number of deep supervision steps being equivalent to a deeper model. On the other hand, with pre-norm and residual scaling, an increase in the number of deep supervision steps does not translate into a significant drop in performance. This observation is in line with our view of looped Transformers as very deep Transformers. 

Note that the setting of the experiment presented in~\Cref{fig:trm_deep_superv_vs_LH} differ from the experiment in~\Cref{tab:trm_ablation}, as in the former the number of deep supervision steps is only modified during inference, while the training procedure is identical to the one used by~\citet{jolicoeurmartineau2025morerecursivereasoningtiny}. However, here we modify the number of deep supervision steps during both training and inference, thus introducing overfitting as a factor in the results. 

\begin{table}[t]
\caption{
    \textbf{Depth, normalization, residual scaling, and hierarchy in TRM.}
    Test sequence accuracy on Sudoku-Extreme (mean $\pm$ standard deviation over three seeds) for TRM with either post-norm or pre-norm with residual scaling, each in its hierarchical and non-hierarchical (single-loop) variant. The indicated number of deep supervision steps is used during both training and inference. LR and WD denote learning rate and weight decay. The post-norm rows use the optimizer settings of the original TRM, for which~\citet{jolicoeurmartineau2025morerecursivereasoningtiny}
    reports 74.7\% in the hierarchical setting. Unlisted hyperparameters follow the original TRM; settings specific to our modifications follow the final FPRM configuration.
    The best mean accuracy is in \textbf{bold}. The worst is \underline{underlined}.}
\label{tab:trm_ablation}
\centering

\begin{tabular}{lcccc}
  \toprule
  TRM variant & Deep sup.\ steps & LR & WD & Accuracy (\%) \\
  \midrule

  \rowcolor{fprmbluelight}
  \multicolumn{5}{l}{\textit{Post-norm}} \\
  \addlinespace[2pt]
  Hierarchical     & 16 & $10^{-4}$ & $1.0$     & $70.44 \pm 3.77$ \\
  Non-hierarchical & 16 & $10^{-4}$ & $1.0$     & $70.13 \pm 7.16$ \\
  Non-hierarchical & 32 & $10^{-4}$ & $1.0$     & $\underline{67.58 \pm 9.69}$ \\
  \midrule

  \rowcolor{fprmbluelight}
  \multicolumn{5}{l}{\textit{Pre-norm + residual scaling}} \\
  \addlinespace[2pt]
  Hierarchical     & 16 & $10^{-3}$ & $10^{-2}$ & {\boldmath$73.08 \pm 5.44$} \\
  Non-hierarchical & 16 & $10^{-4}$ & $1.0$     & $71.65 \pm 3.91$ \\
  Non-hierarchical & 32 & $10^{-4}$ & $1.0$     & $71.28 \pm 3.33$ \\
  \bottomrule
\end{tabular}
\end{table}

\begin{figure}[t]
    \centering
    \includegraphics[width=\linewidth]{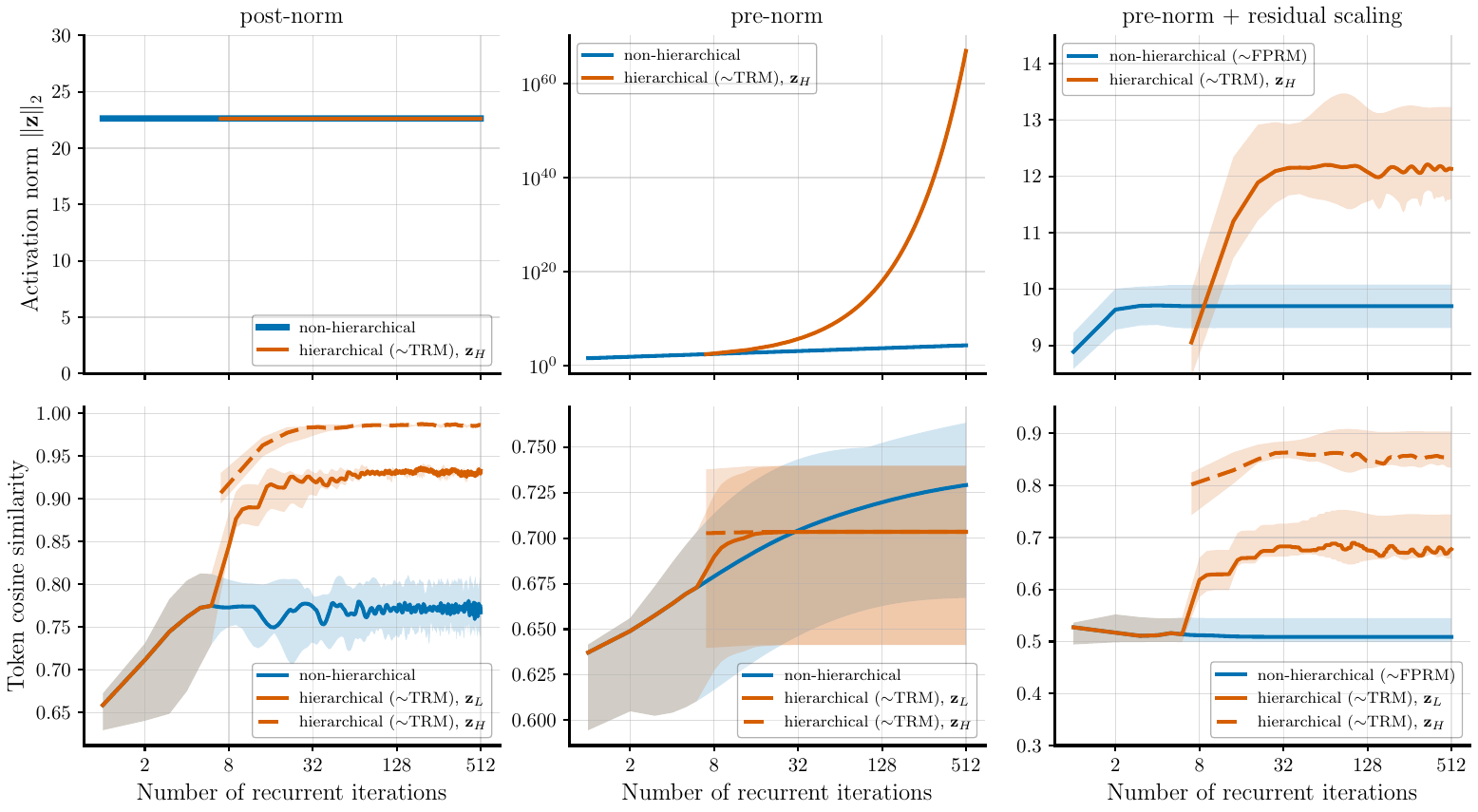}
     \caption{\textbf{Signal propagation in a looped Transformer with and without the hierarchy.}
  A randomly initialized 2-layer Transformer block $f$ is iterated for 512 recurrent iterations from
  $\bmz_0=\mathbf{0}$ on 16 Sudoku-Extreme test puzzles $\bmx$. The non-hierarchical model updates
  $\bmz\leftarrow f(\bmz,\bmx)$; the hierarchical one ($\sim$TRM) alternates six updates
  $\bmz_L\leftarrow f(\bmz_L,\bmz_H+\bmx)$ with one update $\bmz_H\leftarrow f(\bmz_H,\bmz_L)$, using the same weights. Every block call counts as one iteration. Columns: post-norm, pre-norm, and pre-norm with residual scaling
  ($\alpha_1=0.75$, $\alpha_2=0.25$ at initialization). Top: mean per-token $\ell_2$ norm of the state.
  Bottom: mean pairwise cosine similarity between token representations; values approaching 1 indicate rank
  collapse, a symptom of poor signal propagation~\citep{DBLP:conf/icml/DongCL21, DBLP:conf/nips/NociABOSL22}.
  Curves aggregate 5 random initializations (shading: min--max). Computed in float64, since the hierarchical
  pre-norm model overflows in fp32/bf16 after $\approx$300 iterations.}
    \label{fig:trm_signal_prop}
\end{figure}

\paragraph{Activation norm and token similarity at initialization.} To investigate the role of hierarchy in signal propagation in TRM, we look at the activation magnitude and average cosine similarity between the token representations across recurrent iterations in~\Cref{fig:trm_signal_prop}. We consider a looped Transformer with hierarchy ($\sim$ TRM) and without hierarchy, each with post-norm, pre-norm, and pre-norm with residual scaling. We observe that post-norm keeps the activations bounded but lets token similarity grow, most strongly in the hierarchical model ($\bmz_H$ reaches ${\approx}\,0.98$). This is indicative of rank collapse, which also causes vanishing gradients~\citep{DBLP:conf/nips/NociABOSL22}.

Furthermore, pre-norm without residual scaling lets the activations grow without bound, exceeding $10^{60}$ in the hierarchical model after $512$ iterations, far more than in the non-hierarchical model. With pre-norm and residual scaling, the activations stay bounded in both models, and the non-hierarchical model ($\sim$FPRM) has the lowest token similarity of all configurations. This suggests that pre-norm with residual scaling provides bounded activations while also improving the representation diversity of the model. In~\Cref{sec:discussion}, we hypothesize that the hierarchy improves signal propagation. However, the hierarchy does not reduce token similarity: under both post-norm and pre-norm with residual scaling, its states are more similar than those of the non-hierarchical model.

\subsection{Ablation on FPRM}
\label{app:fprm-ablation}

\subsubsection{Normalization and residual scaling}
In this section, we perform an ablation study on FPRM to investigate the impact of each component. In order to save compute, we use a subset of the test-set that we confirm gives an accuracy within the $\pm 0.3\%$ margin of the accuracy on the full test-set. The experiments are performed on the sudoku dataset. We report the mean and the standard deviation over 3 runs. The results are available in~\Cref{tab:fprm_ablation}.

\begin{table}[t]
  \caption{\textbf{Ablation on the components of FPRM.} Test sequence accuracy on Sudoku-Extreme (mean $\pm$ standard deviation over three runs), evaluated on a subset of the test set. Sandwich-norm normalizes both the input and the output of each sub-layer~\citep{geiping2025scalingtesttimecomputelatent}. Residual scales are initialized at $\alpha_1 = 0.75$ and $\alpha_2 = 0.25$. With $\alpha_2$ only, $\alpha_1 = \beta_1 = 1$ and $\beta_2 = 1 - \alpha_2$; without residual scaling, $\alpha_1 = \alpha_2 = \beta_1 = \beta_2 = 1$. The best mean accuracy is in bold. $^\dagger$One of the three runs collapses (14.75\%); the other two reach 82.03\% and 83.50\%.}
  \label{tab:fprm_ablation}
  \centering
  \begin{tabular}{llc}
    \toprule
    Normalization & Residual scaling & Accuracy (\%) \\
    \midrule
    pre-norm        & --                     & diverged \\
    pre-norm        & $\alpha_2$ only        & $65.36 \pm 8.05$ \\
    pre-norm (FPRM) & $\alpha_1, \alpha_2$   & {\boldmath$86.20 \pm 1.19$} \\
    \midrule
    post-norm       & --                     & $81.52 \pm 2.22$ \\
    post-norm       & $\alpha_2$ only        & $60.09 \pm 39.28^\dagger$ \\
    post-norm       & $\alpha_1, \alpha_2$   & $85.00 \pm 2.20$ \\
    \midrule
    sandwich-norm       & --                     & $82.10 \pm 2.34$ \\
    sandwich-norm       & $\alpha_1, \alpha_2$   & $74.93 \pm 2.95$ \\
    \bottomrule
  \end{tabular}
\end{table}

As we can observe, the best performing setting is still pre-norm + residual scaling. However, upon coupling post-norm with residual scaling, the average accuracy improves by $3.5\%$. This can be attributed to residual alleviating signal propagation issues in post-norm, which has been reported in the literature before~\citep{DBLP:conf/nips/NociABOSL22}. Furthermore, we observe that both $\alpha_1$ and $\alpha_2$ are needed in a performant model, which further solidifies their connection to signal propagation problems.

\subsubsection{Damping}
\label{app:damping-ablation}

We evaluate damping during training and/or inference on Sudoku-Extreme
using a subset of the test set, verified to produce accuracies within
$0.3$ percentage points of the full-test-set evaluations.
We report the mean and standard deviation over three runs
in~\Cref{tab:damping-ablation}.

\begin{table}[ht]
    \centering
    \caption{
        \textbf{Ablation on damping.}
        Test sequence accuracy on Sudoku-Extreme
        (mean $\pm$ standard deviation over three runs).
        When damping is enabled at inference, the decay rate is
        $\gamma=0.997$.
    }
    \label{tab:damping-ablation}
    \small
    \begin{tabular}{ccc}
        \toprule
        Train damping & Test damping & Accuracy (\%) \\
        \midrule
        \cmark & \cmark & $\mathbf{90.25 \pm 1.38}$ \\
        \cmark & \xmark & $87.16 \pm 3.90$ \\
        \xmark & \cmark & $83.77 \pm 2.74$ \\
        \xmark & \xmark & $87.23 \pm 2.24$ \\
        \bottomrule
    \end{tabular}
\end{table}

Training with damping yields similar performance with or without
damping at inference, whereas models trained without it do not adapt
well when it is introduced only at inference. The best mean accuracy
is achieved with damping during both training and inference.
We hypothesize that damping provides better optimization signals
from representations closer to the fixed point, motivated
by~\Cref{prop:trunc-bptt-error}.

\subsection{Depth-induced signal propagation issues} \label{app:sig-prop issues}
\subsubsection{Boundedness of activation norms and trainability} \label{sec:boundedness}
As noted in~\Cref{sec: sig_prop}, the normalization scheme governs a trade-off between activation stability and signal propagation, which sharpens with depth. Post-norm keeps activations bounded but suffers from signal propagation issues. Pre-norm enables better signal propagation but lets activations grow without bound. To isolate this trade-off, we adopt the looped Transformer framework~\citep{saunshi2025reasoninglatentthoughtspower, dehghani2019universaltransformers, kaiser2016neuralgpuslearnalgorithms} with a fixed number of \efflayers. This controls for dynamic halting, which can reduce the number of \efflayers, as introduced in~\Cref{sec:method}.

In~\Cref{fig:ut_activation_norm} we show the norm of the final activations of the residual branch of a looped Transformer at initialization. From the Figure, we conclude that the magnitude of the activations grows in the case of a pre-norm looped Transformer. In contrast, when using post-norm or pre-norm with residual scaling, the architecture benefits from bounded activations. In ~\Cref{fig:ut_max_len_vs_iters} we give evidence that boundedness is a prerequisite for reaching deeper \efflayers. The pre-norm architecture with its unbounded activations diverges and can't reach deep \efflayers, while bounded architectures enable trainability. Therefore, a looped Transformer with pre-norm and without residual scaling cannot achieve deep \efflayers, highlighting one aspect of the importance of residual scaling in pre-norm. 

\begin{wrapfigure}[25]{r}{0.3\textwidth}
    \vspace{-2em}
    \centering
    \captionof{table}{Sensitivity of FPRM to residual scaling initialization on the Sudoku-Extreme dataset. Each cell reports the best test sequence accuracy (\%) for a given pair of initial values.}
    \resizebox{\linewidth}{!}{%
      \begin{tabular}{lccc}
        \toprule
        & \multicolumn{3}{c}{$\alpha_2$ init} \\
        \cmidrule(lr){2-4}
        $\alpha_1$ init & 0.25 & 0.50 & 0.75 \\
        \midrule
        0.25 & 83.44 & 78.10 & 83.24 \\
        0.50 & 84.49 & 89.05 & 86.29 \\
        0.75 & \textbf{94.23} & 91.41 & 85.70 \\
        \bottomrule
      \end{tabular}}
    
    \label{tab:alpha_sweep}

    \vspace{1em}

    \includegraphics[width=\linewidth]{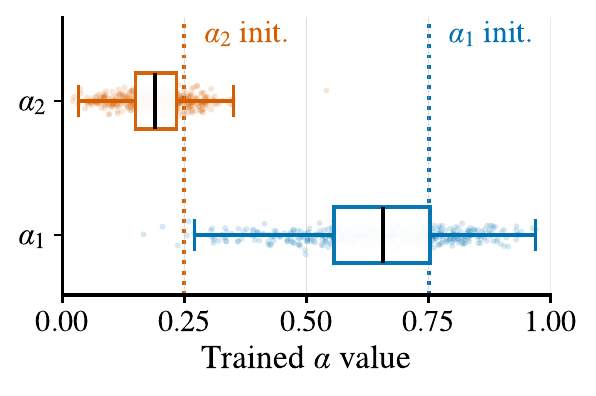}
    \captionof{figure}{The distribution of the residual scales in FPRM after training on the Sudoku-Extreme dataset.}
    \label{fig:alpha_distribution}

\end{wrapfigure}

\subsubsection{The training dynamics of residual scales} \label{sec:alphas_analysis}
%\vspace{0.5em}
While our original motivation for residual scaling was to prevent unbounded activations in pre-norm, we note a parallel relationship between our solution and common solutions to signal propagation problems~\citep{DBLP:journals/corr/abs-2502-05795, DBLP:conf/nips/NociABOSL22}. Specifically, it has been known that scaling down the output of the sub-layers when introducing them to the residual stream is beneficial to increasing \efflayer, which is equivalent to increasing $\alpha_1$ in~\Cref{eq:layer_scaling}. Moreover, in~\Cref{thm:damped-fp-convergence}, we show that the looping will become more stable for smaller $\alpha_2$ in~\Cref{eq:iter_mixing}. In order to investigate the impact of these two parameters, we perform a coarse-grained ablation on the initial values of $\alpha_1$ and $\alpha_2$ on the Sudoku-Extreme task.

\Cref{tab:alpha_sweep} demonstrates that the best initialization places $\alpha_1$ at a high value and $\alpha_2$ at a low value. Interestingly, the $\alpha_1$ preference matches the common solutions to signal propagation problems: keeping the residual stream dominant. Moreover, the ablation also highlights the importance of having a more contractive mapping at initialization, as our convergence analysis (\Cref{thm:damped-fp-convergence}) requires a sufficiently small $\alpha_2$ for the loop to reach a fixed-point. On the other hand, comparing the column with the smallest $\alpha_2$ choice ($\alpha_2=0.25$) with the row with the largest $\alpha_1$ ($\alpha_1=0.75$), we observe that increasing $\alpha_1$ has a slightly higher positive impact on the accuracy than decreasing $\alpha_2$. We hypothesize that this is because it is easier for the model to recover from a bad choice of $\alpha_2$ than $\alpha_1$, as the gradients for $\alpha_1$ come from two different sources (the MHA and MLP sub-layers), and thus can be noisier.

In~\Cref{fig:alpha_distribution}, we provide the distribution statistics of the residual scales after training. Interestingly, we observe that while the median of the $\alpha_1,\alpha_2$ values over channels does not deviate significantly from the initial point, the spread of the distributions widens significantly. In the case of $\alpha_1$, the widening happens at a much larger scale, covering both very small and very large values. On the other hand, the $\alpha_2$ distribution remains more concentrated, with the majority of the values actually becoming smaller than the initial point. This can be interpreted as the model learning to become more contractive during training, which is in line with the observations in~\citet{DBLP:conf/nips/BansalSBEHGG22}. Furthermore, this observation also supports our hypothesis that it might be easier for the model to learn the optimal $\alpha_2$ values than the $\alpha_1$ values.

\section{Additional Experimental Details}\label{experimental details}

\paragraph{Weight initialization.} It seems that initializing the weights using a truncated normal distribution (LeCun initialization) is common practice in looped architectures. In our experiments, it accelerates the convergence but there is very little material difference in sequence accuracy after convergence.

\paragraph{Grokking.} There is some evidence for grokking in looped architectures, but on the maze task we observe convergence on the training data. And training the models for a longer period (up to 7 days) did not yield better performance. 

\paragraph{Hyperparameters and device specifications.}
We provide the values for some of the most important hyperparameters in the paper, for each model and dataset.

\paragraph{Hyperparameter optimization.} For each choice of normalization (pre-norm, post-norm) and residual scaling (with and without) in FPRM and TRM, we optimize the learning rate and the weight decay. The learning rate is selected from the range $\{3e-5, 1e-4, 3e-4, 1e-3, 1.5e-3, 1.7e-3, 3e-3\}$ and the weight-decay is selected from the range $\{1e-3,1e-2,1.0\}$. During the learning rate and weight-decay sweeps, we initialize $\alpha_1=0.75$ and $\alpha_2=0.25$. We also optimize the initial value of $\alpha_1,\alpha_2$ by sweeping both of them from the range $\{0.25, 0.5, 0.75\}$. For both TRM and FPRM, we also compare the following convolution options: a 1D convolution with kernel size 4, a 2D convolution with kernel size $3\times 3$, and no convolution. For FPRM, we also ablate over the maximum number of fixed-point iterations during training, where we either randomly sample the value from an exponential distribution or sweep a fixed value from $\{4,8,10,12,16,32,64,512,1000\}$. For TRM, we use the default values. Furthermore, the number of backward passes is set to $6$ for FPRM following TRM. 

\paragraph{Additional limitations.} In preliminary experiments, increasing the depth of the looped Transformer from two to four Transformer blocks led to substantial overfitting. Jacobian regularization~\citep{bai2021stabilizing} mitigated this issue, but the larger model did not clearly outperform the two-block model. However, these observations are specific to our experimental setup, and whether FPRM benefits from scaling model capacity remains an open question. Furthermore, our fixed-point convergence guarantees rely on assumptions that are not explicitly enforced during training. In practice, the solver may terminate upon reaching its iteration cap or minimum step size without satisfying the residual tolerance.

\begin{table}[h!]
\centering
\caption{Hyperparameters for Sudoku-Extreme experiments (Table~1 of the paper).
Shared across all models: 1$\times$A100-40GB, batch 768, 60\,000 epochs,
constant LR after 2\,000-step warm-up, EMA enabled (rate 0.999), puzzle-embedding length 16. All models are trained with \texttt{AdamW}~\citep{DBLP:conf/iclr/LoshchilovH19}.}
\label{tab:hp-sudoku}
\small
\begin{tabular}{lcc}
\toprule
 & TRM & FPRM \\
\midrule
\multicolumn{3}{l}{\emph{Looping structure}} \\
$H$-cycles            & 3            & --   \\
$L$-cycles            & 6            & --   \\
$H$-layers            & 0            & 0    \\
$L$-layers            & 2            & 2    \\
$n_{\text{back}}$   & $=L\text{-cycles} + 1$             & 6    \\
\midrule
\multicolumn{3}{l}{\emph{Halting}} \\
mechanism                          & ACT  & fixed-point  \\
\texttt{halt\_max\_steps}          & 16   & --           \\
\texttt{max\_iter} (train)         & --   & 12           \\
\texttt{max\_iter} (eval)          & --   & 35\,000      \\
stepsize-decay / patience (train)   & --   & 0.9 / 5   \\
stepsize-decay / patience (eval)   & --   & 0.997 / 10   \\
$\texttt{fp\_thresh}$              & --   & 0.1          \\
\midrule
\multicolumn{3}{l}{\emph{Block / signal-prop modifications}} \\
norm type             & post-norm & pre-norm           \\
residual scaling      & \xmark        & \cmark  \\
$\alpha_1,\alpha_2$ init & --        & 0.75,\;0.25        \\
conv branch           & --      & 2D-conv ($3\times 3$ kernel)   \\
\midrule
\multicolumn{3}{l}{\emph{Optimizer}} \\
learning rate         & $10^{-4}$            & $10^{-3}$            \\
weight decay          & 1.0                  & $10^{-3}$            \\
puzzle-emb LR         & $10^{-4}$ & $10^{-3}$  \\
puzzle-emb WD         & 1.0                  & $10^{-3}$            \\
\bottomrule
\end{tabular}
\end{table}

\begin{table}[h!]
\centering
\caption{Hyperparameters for Maze-Hard experiments (Table~1 of the paper).
Shared: trained on \texttt{maze-30x30-hard-1k} \emph{without augmentation},
4$\times$A100-80GB, constant LR after a 2\,000-step warm-up, EMA enabled (rate 0.999),
puzzle-embedding length 16.
FPRM trains for 60\,000 epochs (TRM 50\,000). FPRM is trained using \texttt{Adam-Atan2}~\citep{DBLP:conf/icml/EverettXWANLGSK24}; TRM is trained using \texttt{AdamW}.}
\label{tab:hp-maze}
\small
\begin{tabular}{lcc}
\toprule
 & TRM & FPRM \\
\midrule
\multicolumn{3}{l}{\emph{Looping structure}} \\
$H$-cycles            & 3            & -- \\
$L$-cycles            & 4            & -- \\
$H$-layers            & 0            & 0 \\
$L$-layers            & 2            & 2 \\
$n_{\text{back}}$     & $=L\text{-cycles} + 1$   & 6 \\
\midrule
\multicolumn{3}{l}{\emph{Halting}} \\
mechanism                          & ACT  & fixed-point \\
\texttt{halt\_max\_steps}          & 16   & --          \\
\texttt{max\_iter} (train)         & --   & 24          \\
\texttt{max\_iter} (eval)          & --   & 35\,000     \\
stepsize-decay / patience (train)   & --   & 0.9 / 5  \\
stepsize-decay / patience (eval)   & --   & 0.996 / 10  \\
\texttt{fp\_thresh}                & --   & 0.1         \\
\midrule
\multicolumn{3}{l}{\emph{Block / signal-prop modifications}} \\
norm type             & post-norm & pre-norm \\
residual scaling      & \xmark    & \cmark   \\
$\alpha_1,\alpha_2$ init & --        & 0.75,\;0.25 \\
conv branch           & --      & 1D-conv ($1\times 4$ kernel) \\
\midrule
\multicolumn{3}{l}{\emph{Optimizer}} \\
learning rate         & $10^{-4}$ & $10^{-4}$ \\
weight decay          & 1.0       & 1.0       \\
puzzle-emb LR         & $10^{-4}$ & $10^{-2}$ \\
puzzle-emb WD         & 1.0       & 1.0       \\
\bottomrule
\end{tabular}
\end{table}

\begin{table}[h!]
\centering
\caption{Hyperparameters for state-tracking experiments on $A_5$ and $S_5$
(Figure~\ref{fig:state-tracking-per-k} of the paper) and for the Looped Transformer signal-propagation
analysis (Figure~\ref{fig:ut_norm}). Shared: 1$\times$A100-80GB, global batch 1024,
\texttt{Adam-Atan2}, no LR warm-up,
EMA disabled, no puzzle embedding (\texttt{puzzle\_emb\_len}=0).
Trained at $k_{\text{train}}{=}32$, evaluated for $k\in[2,128]$.
TRM and FPRM train for 50 epochs; the Looped Transformer analysis (Fig.~2) for 30.}
\label{tab:hp-statetrack}
\small
\setlength{\tabcolsep}{4pt}   %
\begin{tabular}{lccc}
\toprule
 & TRM & FPRM & Looped Transformer (Fig.~2) \\
\midrule
\multicolumn{4}{l}{\emph{Looping structure}} \\
$H$-cycles            & 2                      & --            & -- \\
$L$-cycles            & 4                      & --            & -- \\
$H$-layers            & 0                      & 0             & 0 \\
$L$-layers            & 2                      & 2             & 2 \\
$n_{\text{back}}$     & $=L\text{-cycles} + 1$ & 4             & 4 \\
\midrule
\multicolumn{4}{l}{\emph{Halting}} \\
mechanism             & fixed iters or ACT (inference) & fixed-point   & fixed iters \\
\texttt{halt\_max\_steps}     & 16          & --            & -- \\
\texttt{max\_iter}            & --          & 128           & $=k_{\text{train}}$ \\
iter.\ distribution           & --          & deterministic & deterministic \\
\midrule
\multicolumn{4}{l}{\emph{Block / signal-prop modifications}} \\
norm type             & post-norm   & pre-norm      & sweep$^\dagger$ \\
norm placement        & --          & none          & none \\
residual scaling      & \xmark      & \cmark        & sweep$^\dagger$ \\
$\alpha_1,\alpha_2$ init & --       & 0.5,\;0.5     & sweep$^\dagger$ \\
conv branch           & --          & 1D-conv ($1\times 4$ kernel) & -- \\
\midrule
\multicolumn{4}{l}{\emph{Optimizer}} \\
learning rate         & $10^{-4}$   & $10^{-4}$     & $10^{-4}$ \\
weight decay          & $10^{-2}$   & $10^{-2}$     & $10^{-2}$ \\
\bottomrule
\end{tabular}
\\[2pt]
{\footnotesize $^\dagger$ The Looped Transformer row sweeps the
\{post-norm, pre-norm, pre-norm + residual-scaling\} variants from
\Cref{fig:ut_max_len_vs_iters}; the residual-scaling variant uses
$\alpha_1{=}0.75,\alpha_2{=}0.5$ and $k_{\text{train}}\in\{8,16,32,64\}$.}
\end{table}

\begin{table}[h!]
  \centering
  \caption{Hyperparameters for the FPRM ARC-AGI experiments.
  The two runs share an identical configuration and differ only in the training
  corpus (ARC1-Concept vs.\ ARC2-Concept, both with 1000 augmentations per sample).
  Shared across both: 4$\times$A100-80GB, batch 768, 100\,000 epochs,
  constant LR with no warm-up, EMA enabled (rate 0.999),
  puzzle-embedding length 16, hidden size 512, 8 heads, MLP expansion 4,
  RoPE position encodings. Both models are trained with
  \texttt{Adam-Atan2} ($\beta_1{=}0.9,\beta_2{=}0.95$).}
  \label{tab:hp-arc}
  \small
  \begin{tabular}{lc}
  \toprule
   & FPRM \\
  \midrule
  \multicolumn{2}{l}{\emph{Looping structure}} \\
  $H$-cycles            & --   \\
  $L$-cycles            & --   \\
  $H$-layers            & 0    \\
  $L$-layers            & 2    \\
  $n_{\text{back}}$     & 6    \\
  \midrule
  \multicolumn{2}{l}{\emph{Halting}} \\
  mechanism                          & fixed-point  \\
  \texttt{halt\_max\_steps}          & --           \\
  \texttt{max\_iter} (train)         & 8            \\
  \texttt{max\_iter} (eval)          & 1000 \\
  stepsize / decay / patience (eval) & 1.0 / 0.9 / 5 \\
  $\texttt{fp\_thresh}$              & 0.1          \\
  \midrule
  \multicolumn{2}{l}{\emph{Block / signal-prop modifications}} \\
  norm type             & pre-norm    \\
  residual scaling      & \cmark      \\
  $\alpha_1,\alpha_2$ init & 0.75,\;0.25 \\
  conv branch           & 1D-conv (kernel 4) \\
  \midrule
  \multicolumn{2}{l}{\emph{Optimizer}} \\
  learning rate         & $10^{-3}$   \\
  weight decay          & $10^{-2}$   \\
  puzzle-emb LR         & $10^{-2}$   \\
  puzzle-emb WD         & 1.0         \\
  \bottomrule
  \end{tabular}
  \end{table}
\clearpage

\end{document}